\documentclass{article}

\usepackage{microtype}
\usepackage{graphicx}
\usepackage{subfigure}
\usepackage{booktabs} 

\usepackage{hyperref}

\usepackage[accepted]{icml/icml2018}

\usepackage{enumerate}

\usepackage{amsmath}
\usepackage{amsthm}
\usepackage{amssymb}
\usepackage{thmtools,thm-restate}
\usepackage{amsfonts}
\usepackage{mathtools}

\usepackage{tcolorbox}
\newtheorem{assumption}{Assumption}

\newtheorem{lemma}{Lemma}

\usepackage{mathtools}
\DeclarePairedDelimiter{\abs}{\lvert}{\rvert}
\DeclarePairedDelimiter{\norm}{\lVert}{\rVert}
\DeclarePairedDelimiter{\sq}{[}{]}
\DeclarePairedDelimiter{\br}{(}{)}

\usepackage{chngcntr}
\usepackage{apptools}
\AtAppendix{\counterwithin{lemma}{section}}
\AtAppendix{\counterwithin{theorem}{section}}
\AtAppendix{\counterwithin{algorithm}{section}}
\AtAppendix{\counterwithin{figure}{section}}
\newcommand{\Signum}{\textsc{Signum}}
\newcommand{\signSGD}{\textsc{signSGD}}
\newcommand{\SGD}{\textsc{SGD}}
\newcommand{\Adam}{\textsc{Adam}}
\newcommand{\RMSprop}{\textsc{RMSprop}}
\newcommand{\Rprop}{\textsc{Rprop}}
\newcommand{\QSGD}{\textsc{QSGD}}
\newcommand{\TernGrad}{\textsc{TernGrad}}
\newcommand{\Adagrad}{\textsc{Adagrad}}

\newcommand{\SNR}{S}
\newcommand{\Lnorm}{\|\vec{L}\|_1}

\newcommand{\R}{\mathbb{R}}

\def\E{\mathbb{E}}
\def\P{\mathbb{P}}

\def\diag{\mathrm{diag}}

\def\R{\mathbb{R}}

\begin{document}

\twocolumn[
\icmltitlerunning{\signSGD{}: Compressed Optimisation for Non-Convex Problems}
\icmltitle{\signSGD{}: Compressed Optimisation for Non-Convex Problems}




\icmlsetsymbol{equal}{*}

\begin{icmlauthorlist}
\icmlauthor{Jeremy Bernstein}{caltech,amz}
\icmlauthor{Yu-Xiang Wang}{amz,ucsb}
\icmlauthor{Kamyar Azizzadenesheli}{irvine}
\icmlauthor{Anima Anandkumar}{caltech,amz}
\end{icmlauthorlist}

\icmlaffiliation{amz}{Amazon AI}
\icmlaffiliation{caltech}{Caltech}
\icmlaffiliation{irvine}{UC Irvine}
\icmlaffiliation{ucsb}{UC Santa Barbara}

\icmlcorrespondingauthor{Jeremy Bernstein}{bernstein@caltech.edu}
\icmlcorrespondingauthor{Yu-Xiang Wang}{yuxiangw@amazon.edu}

\icmlkeywords{non-convex, stochastic, distributed, quantized, sign, optimization, deep learning}

\vskip 0.3in
]



\printAffiliationsAndNotice{}  

\begin{abstract} 
Training large neural networks requires distributing learning across multiple workers, where the cost of communicating gradients can be a significant bottleneck.  \signSGD{} alleviates this problem by  transmitting just the sign of each minibatch stochastic gradient. We prove that it can get the best of both worlds: compressed gradients \emph{and} \SGD{}-level convergence rate. The relative $\ell_1/\ell_2$ geometry of gradients, noise and curvature informs whether \signSGD{} or \SGD{} is theoretically better suited to a particular problem.
On the practical side we find that the momentum counterpart of \signSGD{} is able to match the accuracy and convergence speed of \Adam{} on deep Imagenet models. We extend our theory to the distributed setting, where the parameter server uses majority vote to aggregate gradient signs from each worker enabling 1-bit compression of worker-server communication \emph{in both directions}. Using a theorem by \citet{gauss} we prove that majority vote can achieve the same reduction in variance as full precision distributed \SGD{}. Thus,  there is great promise for sign-based optimisation schemes to achieve fast communication \emph{and} fast convergence. Code to reproduce experiments is to be found at \url{https://github.com/jxbz/signSGD}.

\end{abstract}
\section{Introduction}
\label{sec:intro}
Deep neural networks have learnt to solve numerous natural human tasks \cite{mafia,youagain}. Training these large-scale models can take days or even weeks. The learning process can be accelerated by distributing training over multiple processors---either GPUs linked within a single machine, or even multiple machines linked together. Communication between workers is typically handled using a parameter-server framework \cite{muli}, which involves repeatedly communicating the gradients of every parameter in the model. This can still be time-intensive for  large-scale neural networks. The communication cost can be reduced if gradients are   compressed before being transmitted.  In this paper, we analyse the theory of robust schemes for gradient compression.

\begin{algorithm}[t]
   \caption{\signSGD{}}
   \label{alg:signSGD}
\begin{algorithmic}
   \STATE {\bfseries Input:} learning rate $\delta$, current point $x_k$
   \STATE $\tilde{g}_k \leftarrow \mathrm{stochasticGradient}(x_k)$
   \STATE $x_{k+1} \leftarrow x_k - \delta \, \mathrm{sign}(\tilde{g}_k)$
\end{algorithmic}
\end{algorithm}

\begin{algorithm}[tb]
   \caption{\Signum{}}
   \label{alg:signum}
\begin{algorithmic}
   \STATE {\bfseries Input:} learning rate $\delta$, momentum constant $\beta \in (0,1)$, current point $x_k$, current momentum $m_k$

   \STATE $\tilde{g}_k \leftarrow \mathrm{stochasticGradient}(x_k)$   
   \STATE $m_{k+1} \leftarrow \beta m_k + (1-\beta) \tilde{g}_k$
   \STATE $x_{k+1} \leftarrow x_k - \delta \,\mathrm{sign}(m_{k+1})$
\end{algorithmic}
\end{algorithm}

\newlength\myindent
\setlength\myindent{2em}
\newcommand\bindent{%
  \begingroup
  \setlength{\itemindent}{\myindent}
  \addtolength{\algorithmicindent}{\myindent}
}
\newcommand\eindent{\endgroup}

\begin{algorithm}[tb]
   \caption{Distributed training by majority vote}
   \label{alg:majority}
\begin{algorithmic}
   \STATE {\bfseries Input:} learning rate $\delta$, current point $x_k$, \# workers $M$ each with an independent gradient estimate $\tilde{g}_m(x_k)$
   \STATE \textbf{on} server
   \bindent
   \STATE \textbf{pull} $\textrm{sign}(\tilde{g}_m)$ \textbf{from} each worker
	\STATE \textbf{push} $\textrm{sign}\sq*{\sum_{m=1}^M\textrm{sign}(\tilde{g}_m)}$ \textbf{to} each worker 
   \eindent

	\STATE \textbf{on} each worker
   \bindent
	\STATE $x_{k+1} \leftarrow x_k - \delta \, \textrm{sign}\sq*{\sum_{m=1}^M\textrm{sign}(\tilde{g}_m)}$
   \eindent

\end{algorithmic}
\end{algorithm}

An elegant form of gradient compression is just to take the sign of each coordinate of the stochastic gradient vector, which we call \signSGD{}. The algorithm is as simple as throwing away the exponent and mantissa of a 32-bit floating point number. Sign-based methods have been studied at least since the days of \Rprop{} \cite{riedmiller_direct_1993}. This algorithm inspired many popular optimisers---like \RMSprop{} \cite{tieleman_rmsprop_2012} and \Adam{} \cite{kingma_adam:_2015}. But researchers were interested in \Rprop{} and variants because of their robust and fast convergence, and not their potential for gradient compression.  

Until now there has been no rigorous theoretical explanation for the  empirical success of sign-based stochastic gradient  methods. The sign of the stochastic gradient  is a biased approximation to the true gradient, making it more challenging to analyse compared to standard \SGD{}.  In this paper, we provide extensive theoretical analysis  of  sign-based methods for non-convex optimisation under transparent assumptions. We show that \signSGD{} is especially efficient in problems with a particular $\ell_1$ geometry: when gradients are as dense or denser than stochasticity and curvature, then \signSGD{} can converge with a theoretical rate that has similar or even better dimension dependence than \SGD{}. We find empirically that \emph{both gradients and noise are dense} in deep learning problems, consistent with the observation that \signSGD{} converges at a similar rate to \SGD{} in practice.

We then analyse \signSGD{} in the distributed setting where the parameter server aggregates gradient signs of the workers by a majority vote. Thus we allow worker-server communication to be 1-bit compressed in both directions. We prove that the theoretical speedup matches that of distributed \SGD{}, under natural assumptions that are validated by experiments. 

We also extend our theoretical framework to the \Signum{} optimiser---which takes the \underline{\smash{sign}} of the moment\underline{\smash{um}}. Our theory suggests that momentum may be useful for controlling a tradeoff between bias and variance in the estimate of the stochastic gradient. On the practical side, we show that \Signum{} easily scales to large Imagenet models, and provided the learning rate and weight decay are tuned, all other hyperparameter settings---such as momentum, weight initisialiser, learning rate schedules and data augmentation---may be lifted from an \SGD{} implementation.
\section{Related Work}
\label{sec:related}
\textbf{Distributed machine learning:} From the information theoretic angle, \citet{mcmahan} study the communication limits of estimating the mean of a general quantity known about only through samples collected from $M$ workers. In contrast, we focus exclusively on communication of gradients for optimisation, which allows us to exploit the fact that we do not care about incorrectly communicating small gradients in our theory. Still our work has connections with information theory. When the parameter server aggregates gradients by majority vote, it is effectively performing maximum likelihood decoding of a repetition encoding of the true gradient sign that is supplied by the M workers.

As for existing gradient compression schemes, \citet{seide_1-bit_2014} and \citet{strom2015scalable} demonstrated empirically that 1-bit quantisation can still give good performance whilst dramatically reducing gradient communication costs in distributed systems. \citet{QSGD} and \citet{wen2017terngrad} provide schemes with theoretical guarantees by using random number generators to ensure that the compressed gradient is still an unbiased approximation to the true gradient. Whilst unbiasedness allows these schemes to bootstrap \SGD{} theory, it unfortunately comes at the cost of hugely inflated variance, and this variance explosion\footnote{For the version of \QSGD{} with 1-bit compression, the variance explosion is by a factor of $\sqrt{d}$, where $d$ is the number of weights. It is common to have $d > 10^8$ in modern deep networks.} basically renders the \SGD{}-style bounds vacuous in the face of the empirical success of these algorithms. The situation only gets worse when the parameter server must aggregate and send back the received gradients, since merely summing up quantised updates reduces the quantisation efficiency. We compare the schemes in Table \ref{tab:compare}---notice how the existing schemes pick up log factors in the transmission from parameter-server back to workers. Our proposed approach is different, in that we directly employ the sign gradient which is \emph{biased}. This avoids the randomisation needed for constructing an unbiased quantised estimate, avoids the problem of variance exploding in the theoretical bounds, and even enables 1-bit compression in both directions between parameter-server and workers, at no theoretical loss compared to distributed \SGD{}.

\begin{table}[t]
\centering
\caption{The communication cost of different gradient compression schemes, when  training a $d$-dimensional model with $M$ workers.}
\label{tab:compare}
\vskip 0.15in
\begin{small}
\begin{tabular}{lr}
\toprule
\textsc{Algorithm} & \textsc{\# bits per iteration} \\
\midrule
\SGD{} \cite{robbins1951} & $64Md$\\
\QSGD{} \cite{QSGD}   & $(2+\log(2M+1))Md$ \\
\TernGrad{} \cite{wen2017terngrad}& $(2+\log(2M+1))Md$ \\
\signSGD{} with majority vote & $2Md$ \\
\bottomrule
\end{tabular}
\end{small}
\vskip -0.1in
\end{table}

\textbf{Deep learning:} stochastic gradient descent \cite{robbins1951} is a simple and extremely effective optimiser for training neural networks. Still \citet{riedmiller_direct_1993} noted the good practical performance of sign-based methods like \Rprop{} for training deep nets, and since then variants such as \RMSprop{} \cite{tieleman_rmsprop_2012} and \Adam{} \cite{kingma_adam:_2015} have become increasingly popular. \Adam{} updates the weights according to the mean divided by the root mean square of recent gradients. Let $\langle . \rangle_\beta$ denote an exponential moving average with timescale $\beta$, and $\tilde{g}$ the stochastic gradient. Then
\begin{align*}
	\text{\Adam{} step } &\sim\frac{\langle \tilde{g}\rangle_{\beta_1}}{\sqrt{\langle \tilde{g}^2\rangle_{\beta_2}}}
\intertext{Therefore taking the time scale of the exponential moving averages to zero, $\beta_1, \beta_2 \rightarrow 0$, yields \signSGD{}}
	\text{\signSGD{} step } &= \text{sign}(\tilde{g})= \frac{\tilde{g}}{\sqrt{\tilde{g}^2}}.
\end{align*}
To date there has been no convincing theory of the \{\Rprop{}, \RMSprop{}, \Adam{}\} family of algorithms, known as `adaptive gradient methods'. Indeed \citet{adam-non-converge} point out problems in the original convergence proof of \Adam{}, even in the convex setting. Since \signSGD{} belongs to this same family of algorithms, we expect that our theoretical analysis should be relevant for all algorithms in the family. In a parallel work, \citet{balles2018dissecting} explore the connection between \signSGD{} and \Adam{} in greater detail, though their theory is more restricted and lives in the convex world, and they do not analyse \Signum{} as we do but employ it on heuristic grounds.

\textbf{Optimisation:} much of classic optimisation theory focuses on convex problems, where \emph{local information} in the gradient tells you \emph{global information} about the direction to the minimum. Whilst elegant, this theory is less relevant for modern problems in deep learning which are non-convex. In non-convex optimisation, finding the global minimum is generally intractable. Theorists usually settle for measuring some restricted notion of success, such as rate of convergence to stationary points \ \citep{ghadimi2013stochastic,allen-zhu_natasha:_2017} or local minima \citep{nesterov_cubic_2006}. Though \citet{dauphin_identifying_2014} suggest saddle points should abound in neural network error landscapes, practitioners report not finding this a problem in practice \cite{goodfellow_qualitatively_2015} and therefore a theory of convergence to stationary points is useful and informative. 

On the algorithmic level, the \emph{non-stochastic} version of \signSGD{} can be viewed as the classical steepest descent algorithm with $\ell_\infty$-norm \citep[see, e.g.,][Section 9.4]{boyd2004convex}.   The convergence of steepest descent is well-known \citep[see][Appendix C, for an analysis of signed gradient updates under the Polyak-\L{}ojasiewicz condition]{Karimi2016LinearCO}. \citet{ssd} study a stochastic version of the algorithm, but again under an $\ell_\infty$ majorisation. To the best of our knowledge, we are the first to study the convergence of signed gradient updates under an (often more natural) $\ell_2$ majorisation (Assumption \ref{a:coordinate_lip}).


\textbf{Experimental benchmarks:} throughout the paper we will make use of the CIFAR-10 \cite{Krizhevsky09learningmultiple} and Imagenet \cite{ILSVRC15} datasets. As for neural network architectures, we train Resnet-20 \cite{he2016deep} on CIFAR-10, and Resnet-50 v2 \cite{resnetv2} on Imagenet.


\section{Convergence Analysis of \signSGD{}}
\label{sec:theorysection}

We begin our analysis of sign stochastic gradient descent in the non-convex setting. The standard assumptions of the stochastic optimisation literature are nicely summarised by \citet{allen-zhu_natasha_2017}. We will use more fine-grained assumptions. \signSGD{} can exploit this additional structure, much as \Adagrad{} \cite{adagrad} exploits sparsity. We emphasise that these fine-grained assumptions do not lose anything over typical \SGD{} assumptions, since \emph{our assumptions can be obtained from SGD assumptions and vice versa}.

\begin{assumption}[Lower bound]\label{a:lower}
  For all $x$ and some constant $f^*$, we have objective value
   $   f(x) \geq f^*.$
\end{assumption}

This assumption is standard and necessary for guaranteed convergence to a stationary point. 

The next two assumptions will naturally encode notions of heterogeneous curvature and gradient noise.

  \begin{assumption}[Smooth]\label{a:coordinate_lip}
   Let $g(x)$ denote the gradient of the objective $f(.)$ evaluated at point $x$. Then $\forall x, y$ we require that for some non-negative constant $\vec L := [L_1,...,L_d]$
  \begin{align*}
\abs[\Big]{f(y) - \sq*{f(x) + g(x)^T(y-x)}} &\leq \frac{1}{2}\sum_i L_i (y_i-x_i)^2.
  \end{align*}
  \end{assumption}
  For twice differentiable $f$, this implies that $ -\text{diag}(\vec L) \prec H\prec \text{diag}(\vec L)$. This is related to the slightly weaker coordinate-wise Lipschitz condition used in the block coordinate descent literature \citep{richtarik2014iteration}. 
    
  Lastly, we assume that we have access to the following stochastic gradient oracle:
  \begin{assumption}[Variance bound]\label{a:hoeffding}
	Upon receiving query $x \in \mathbb{R}^d$, the stochastic gradient oracle gives us an \emph{independent} unbiased estimate $\tilde{g}$ that has coordinate bounded variance:
	\begin{equation*}
	\mathbb{E}[\tilde{g}(x)]= g(x), \qquad   \E\left[(\tilde{g}(x)_i-g(x)_i)^2\right]  \leq \sigma_i^2
	\end{equation*}
    for a vector of non-negative constants $\vec \sigma := [\sigma_1, .., \sigma_d]$.
\end{assumption}

Bounded variance may be unrealistic in practice, since as $x \rightarrow \infty$ the variance might well diverge. Still this assumption is useful for understanding key properties of stochastic optimisation algorithms. In our theorem, we will be working with a mini-batch of size $n_k$ in the $k^{th}$ iteration, and the corresponding mini-batch stochastic gradient is modeled as the average of $n_k$ calls to the above oracle at $x_k$. This squashes the variance bound on $\tilde{g}(x)_i$ to $\sigma_i^2/n_k$.

Assumptions~\ref{a:coordinate_lip}~and~\ref{a:hoeffding} are different from the assumptions typically used for analysing the convergence properties of \SGD{} \citep{nesterov2013introductory,ghadimi2013stochastic}, but they are natural to the geometry induced by algorithms with signed updates such as \signSGD{} and \Signum{}.

Assumption~\ref{a:coordinate_lip} is more fine-grained than the standard assumption, which is recovered by defining $\ell_2$ Lipschitz constant $L := \|\vec L\|_\infty = \max_{i} L_i$. Then Assumption \ref{a:coordinate_lip} implies that
  \begin{align*}
\abs[\Big]{f(y) - \sq*{f(x) + g(x)^T(y-x)}} &\leq \frac{L}{2}\|y_i-x_i\|_2^2.
  \end{align*}
  which is the standard assumption of Lipschitz smoothness.

Assumption~\ref{a:hoeffding} is  more fined-grained than the standard stochastic gradient oracle assumption used for \SGD{} analysis. But again, the standard variance bound is recovered by defining $\sigma^2 := ||\vec \sigma||_2^2$. Then Assumption \ref{a:hoeffding} implies that
	\begin{equation*}
\E\|\tilde{g}(x)-g(x)\|^2  \leq \sigma^2
	\end{equation*}
   which is the standard assumption of bounded total variance.

Under these assumptions, we have the following result:

\begin{tcolorbox}[boxsep=0pt, arc=0pt,
    boxrule=0.5pt,
 colback=white]
\begin{restatable}[Non-convex convergence rate of 
\signSGD{}]{theorem}{signSGDtheorem}\label{thm:signSGD}
Run algorithm \ref{alg:signSGD} for $K$ iterations under Assumptions 1 to 3. Set the learning rate and mini-batch size (independently of step $k$) as
  \begin{equation*}
  \delta_k = \frac{1}{\sqrt{\Lnorm K}}, \qquad \qquad n_k = K
  \end{equation*}
  Let   $N$ be the cumulative number of stochastic gradient calls up to step $K$, i.e.\  $N = \text{O}(K^2)$. Then we have
   \begin{align*}
    &\mathbb{E}\sq*{\frac{1}{K}\sum_{k=0}^{K-1} \norm{g_k}_1}^2 \\
    & \qquad \leq \frac{1}{\sqrt{N}}\sq*{\sqrt{\Lnorm}\br*{f_0 - f_*+\frac{1}{2}} + 2\|\vec\sigma\|_1}^2
    \end{align*}
 
\end{restatable}
\end{tcolorbox}
	The proof is given in Section \ref{app:sign} of the supplementary material. It follows the well known strategy of relating the norm of the gradient to the expected improvement made in a single algorithmic step, and comparing this with the total possible improvement under Assumption \ref{a:lower}. A key technical challenge we overcome is in showing how to directly deal with a biased approximation to the true gradient. Here we will provide some intuition about the proof. 
    
    To pass the stochasticity through the non-linear sign operation in a controlled fashion, we need to prove the key statement that at the $k^{th}$ step for the $i^{th}$ gradient component
    \begin{equation*}
  \mathbb{P}\sq{\text{sign}(\tilde{g}_{k,i}) \neq \text{sign}(g_{k,i})} \leq \frac{\sigma_{k,i}}{|g_{k,i}|}
    \end{equation*}
    This formalises the intuition that the probability of the sign of a component of the stochastic gradient being incorrect should be controlled by the signal-to-noise ratio of that component. When a component's gradient is large, the probability of making a mistake is low, and one expects to make good progress. When the gradient is small compared to the noise, the probability of making mistakes can be high, but due to the large batch size this only happens when we are already close to a stationary point.
  
The large batch size in the theorem may seem unusual, but large batch training actually presents a systems advantage \cite{Goyal2017AccurateLM} since it can be parallelised. The number of gradient calls $N$ is the important quantity to measure convergence, but large batch training achieves $N$ gradient calls in only O$(\sqrt{N})$ iterations whereas small batch training needs $O(N)$ iterations. Fewer iterations also means fewer rounds of communication in the distributed setting. Convergence guarantees \emph{can} be extended to the small batch case under the additional assumption of unimodal symmetric gradient noise using Lemma \ref{lem:symm} in the supplementary, but we leave this for future work. Experiments in this paper were indeed conducted in the small batch regime.
  
        
    Another unusual feature requiring discussion is the $\ell_1$ geometry of \signSGD{}.
    The convergence rate strikingly depends on the $\ell_1$-norm of the gradient, the stochasticity and the curvature. To understand this better, let's define a notion of density of a high-dimensional vector $\vec v \in \mathbb{R}^d$ as follows:
    \begin{equation}
    	\phi(\vec v) := \frac{\|\vec v\|_1^2}{d\|\vec v\|_2^2}
    \end{equation}
    To see that this is a natural definition of density, notice that for a fully dense vector, $\phi(\vec v) = 1$ and for a fully sparse vector, $\phi(\vec v) = 1/d \approx 0$. We trivially have that $\|\vec{v}\|_1^2\leq\phi(\vec{v})d^2\|\vec{v}\|_\infty^2$ so this notion of density provides an easy way to translate from norms in $\ell_1$ to both $\ell_2$ and $\ell_\infty$.
    
    Remember that under our assumptions, \SGD{}-style assumptions hold with Lipschitz constant $L:=\|\vec{L}\|_\infty$ and total variance bound $\sigma^2 := \|\vec\sigma\|_2^2$. Using our notion of density we can translate our constants into the language of \SGD{}:
    \begin{align*}
    \|g_k\|_1^2 &= \phi(g_k)d\|g_k\|_2^2 &&\geq\phi(g)d\|g_k\|_2^2\\
    	\|\vec{L}\|_1^2 &\leq \phi(\vec{L})d^2 \|\vec{L}\|_\infty^2 &&=\phi(\vec{L})d^2 L^2 \\
    	\|\vec{\sigma}\|_1^2 &= \phi(\vec{\sigma})d \|\vec{\sigma}\|_2^2 &&=\phi(\vec{\sigma})d \sigma^2
    \end{align*}
    where we have assumed $\phi(g)$ to be a lower bound on the gradient density over the entire space. Using that $(x+y)^2\leq2(x^2+y^2)$ and changing variables in the bound, we reach the following result for \signSGD{}
    \begin{align*}
 & \mathbb{E}\sq*{\frac{1}{K}\sum_{k=0}^{K-1}\norm{g_k}_2}^2 \\
    & \qquad \leq \frac{2}{\sqrt{N}}\sq*{\frac{\sqrt{\phi(\vec{L})}}{\phi(g)} L\br*{f_0 - f_*+\frac{1}{2}}^2 + 4\frac{\phi(\vec{\sigma})}{\phi(g)} \sigma^2}
    \end{align*}
    whereas, for comparison, a typical \SGD{} bound (proved in Supplementary \ref{app:sgdtheory}) is
    \begin{align*}
    \mathbb{E}\sq*{\frac{1}{K}\sum_{k=0}^{K-1} \norm{g_k}_2^2} \leq \frac{1}{\sqrt{N}}\sq*{2L(f_0 - f_*) + \sigma^2}.
    \end{align*}
The bounds are very similar, except for most notably the appearance of ratios of densities $R_1$ and $R_2$, defined as
\begin{align*}
R_1 := \frac{\sqrt{\phi(\vec{L})}}{\phi(g)} \qquad\qquad R_2 := \frac{\phi(\vec{\sigma})}{\phi(g)}
\end{align*}
Na\"{i}vely comparing the bounds suggests breaking into cases:
\begin{enumerate}[(I)]
\item $R_1 \gg 1$ and $R_2 \gg 1$. This means that both the curvature and the stochasticity are much denser than the typical gradient and the comparison suggests \SGD{} is better suited than \signSGD{}.
\item NOT$[R_1 \gg 1]$ and NOT$[R_2 \gg 1]$. This means that neither curvature nor stochasticity are much denser than the gradient, and the comparison suggests that \signSGD{} may converge as fast or faster than \SGD{}, and also get the benefits of gradient compression.
\item neither of the above holds, for example $R_1 \ll 1$ and $R_2 \gg 1$. Then the comparison is indeterminate about whether \signSGD{} or \SGD{} is more suitable.
\end{enumerate}

Let's briefly provide some intuition to understand how it's possible that \signSGD{} could outperform \SGD{}. Imagine a scenario where the gradients are dense but there is a sparse set of extremely noisy components. Then the dynamics of \SGD{} will be dominated by this noise, and (unless the learning rate is reduced a lot) \SGD{} will effectively perform a random walk along these noisy components, paying less attention to the gradient signal. \signSGD{} however will treat all components equally, so it will scale down the sparse noise and scale up the dense gradients comparatively, and thus make good progress. See Figure \ref{fig:sparse_noise} in the supplementary for a simple example of this.

Still we must be careful when comparing upper bounds, and interpreting the dependence on curvature density is more subtle than noise density. This is because the \SGD{} bound proved in Supplementary \ref{app:sgdtheory} is slacker under situations of sparse curvature than dense curvature. That is to say that \SGD{}, like \signSGD{}, benefits under situations of sparse curvature but this is not reflected in the \SGD{} bound. The potentially slack step in \SGD{}'s analysis is in switching from $L_i$ to $\|\vec L\|_\infty$. Because of this it is safer to interpret the curvature comparison as telling us a regime where \signSGD{} is expected to lose out to \SGD{} (rather than vice versa). This happens when $R_1 \gg 1$ and gradients are sparser than curvature. Intuitively, in this case \signSGD{} will push many components in highly curved directions even though these components had small gradient, and this can be undesirable.

To summarise, our theory suggests that when gradients are dense, \signSGD{} should be more robust to large stochasticity on a sparse set of coordinates. When gradients are sparse, \SGD{} should be more robust to dense curvature and noise. In practice for deep networks, we find that \signSGD{} converges about as fast as \SGD{}. That would suggest that we are either in regime (II) or (III) above. But what is the real situation for the error landscape of deep neural networks? 

      \begin{figure}[t]
\vskip 0.2in
\begin{center}
\centerline{\includegraphics[width=\columnwidth]{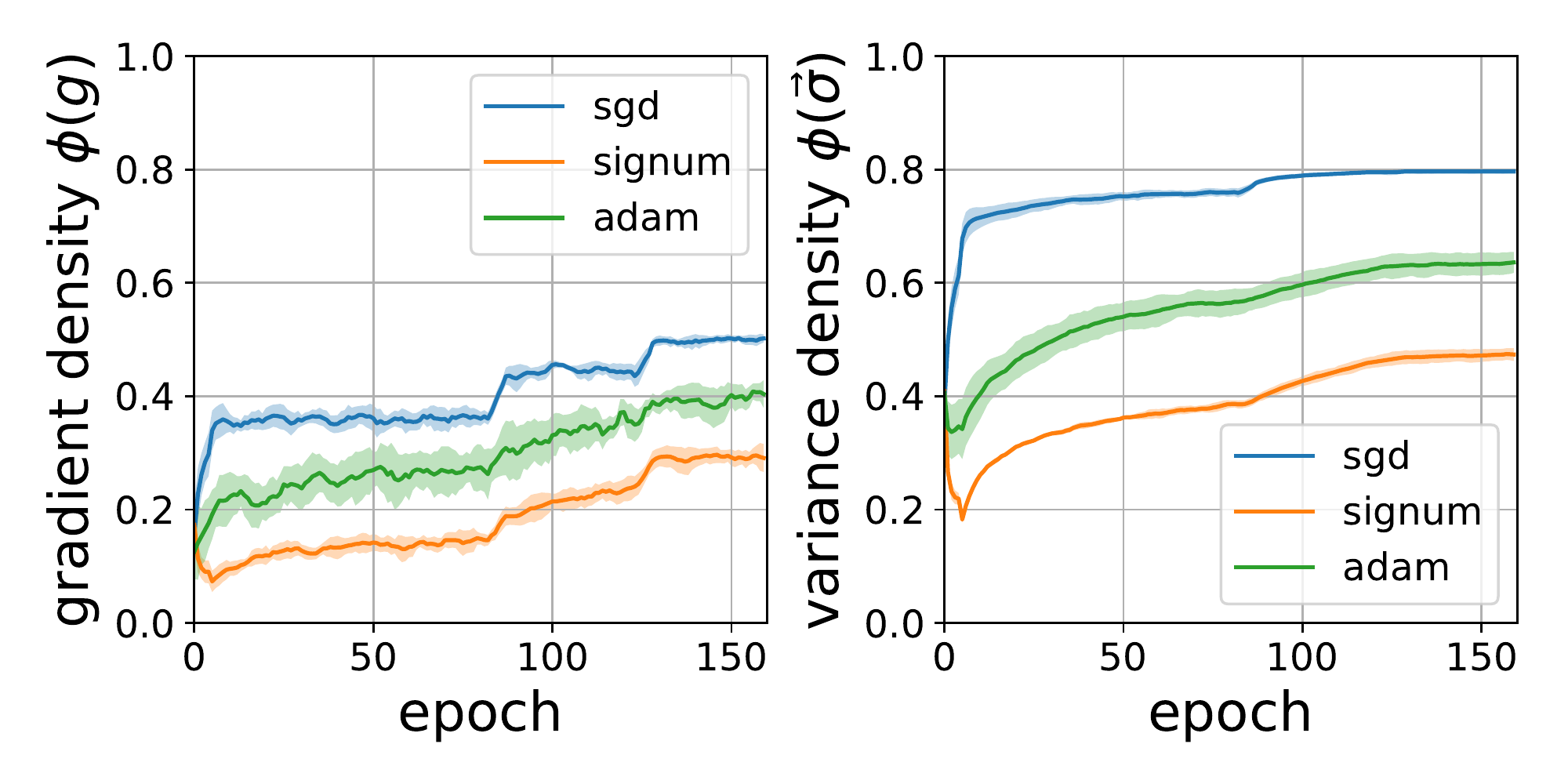}}
\caption{Gradient and noise density during an entire training run of a Resnet-20 model on the CIFAR-10 dataset. Results are averaged over 3 repeats for each of 3 different training algorithms, and corresponding error bars are plotted. At the beginning of every epoch, at that fixed point in parameter space, we do a full pass over the data to compute the exact mean of the stochastic gradient, $g$, and its exact standard deviation vector $\vec{\sigma}$ (square root of diagonal of covariance matrix). The density measure $\phi(\vec{v}):=\frac{\|v\|_1^2}{d\|v\|_2^2}$ is 1 for a fully dense vector and $\approx 0$ for a fully sparse vector. Notice that both gradient and noise are dense, and moreover the densities appear to be coupled during training. Noticable jumps occur at epoch 80 and 120 when the learning rate is decimated. Our stochastic gradient oracle (Assumption \ref{a:hoeffding}) is fine-grained enough to encode such dense geometries of noise.
}
\label{fig:grad-stats}
\end{center}
\vskip -0.2in
\end{figure}

To measure gradient and noise densities in practice, we use Welford's algorithm \cite{welford,knuth} to compute the true gradient $g$ and its stochasticity vector $\vec{\sigma}$ at every epoch of training for a Resnet-20 model on CIFAR-10. Welford's algorithm is numerically stable and only takes a single pass through the data to compute the vectorial mean and variance. Therefore if we train a network for 160 epochs, we make an additional 160 passes through the data to evaluate these gradient statistics. Results are plotted in Figure \ref{fig:grad-stats}. Notice that the gradient density and noise density are of the same order throughout training, and this indeed puts us in regime (II) or (III) as predicted by our theory. 

In Figure \ref{fig:grad-dense} of the supplementary, we present preliminary evidence that this finding generalises, by showing that gradients are dense across a range of datasets and network architectures. We have not devised an efficient means to measure curvature densities, which we leave for future work.
\section{Majority Rule: the Power of Democracy in the Multi-Worker Setting}
\label{sec:majority}

In the most common form of distributed training, workers (such as GPUs) each evaluate gradients on their own split of the data, and send the results up to a parameter-server. The parameter server aggregates the results and transmits them back to each worker \cite{muli}.

\begin{figure}[t]
\vskip 0.2in
\begin{center}
\centerline{\includegraphics[width=\columnwidth]{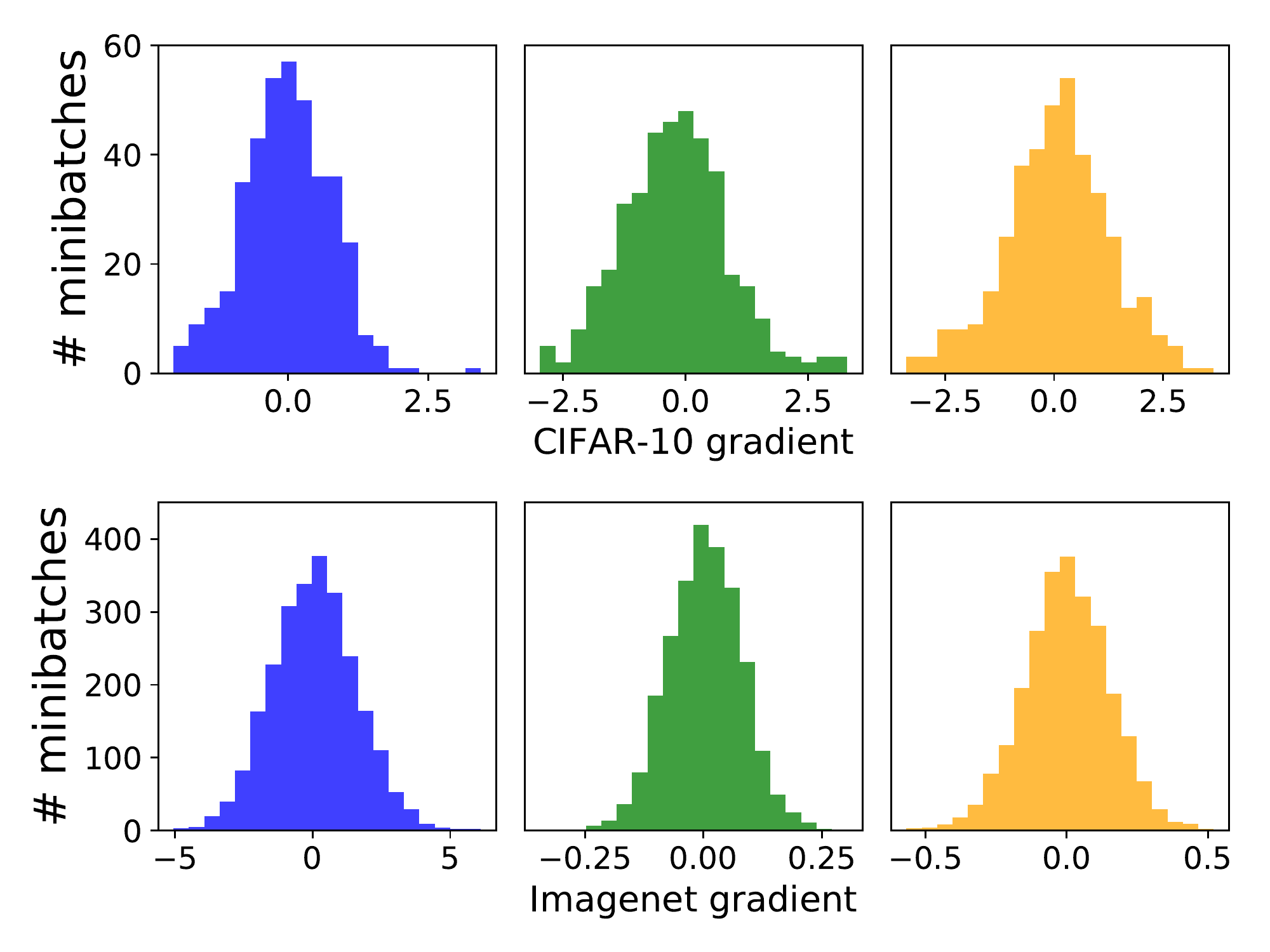}}
\caption{Histograms of the noise in the stochastic gradient, each plot for a different randomly chosen parameter (not cherry-picked). Top row: Resnet-20 architecture trained to epoch 50 on CIFAR-10 with a batch size of 128. Bottom row: Resnet-50 architecture trained to epoch 50 on Imagenet with a batch size of 256. From left to right: model trained with \SGD{}, \Signum{}, \Adam{}. All noise distributions appear to be unimodal and approximately symmetric. For a batch size of 256 Imagenet images, the central limit theorem has visibly kicked in and the distributions look Gaussian.}
\label{fig:symmetry}
\end{center}
\vskip -0.2in
\end{figure}

Up until this point in the paper, we have only analysed \signSGD{} where the update is of the form
\begin{equation*}
	x_{k+1} = x_k - \delta\, \text{sign}(\tilde{g})
\end{equation*}
To get the benefits of compression we want the $m^{th}$ worker to send the sign of the gradient evaluated only on its portion of the data. This suggests an update of the form
\begin{equation*}
	x_{k+1} = x_k - \delta\sum_{m=1}^M \text{sign}(\tilde g_m) \tag{good}
\end{equation*}
This scheme is good since what gets sent \emph{to} the parameter will be 1-bit compressed. But what gets sent back almost certainly will not. Could we hope for a scheme where all communication is 1-bit compressed?

What about the following scheme:
\begin{equation*}
	x_{k+1} = x_k - \delta\,\text{sign}\sq*{\sum_{m=1}^M \text{sign}(\tilde g_m)}\tag{best}
\end{equation*}
This is called majority vote, since each worker is essentially voting with its belief about the sign of the true gradient. The parameter server counts the votes, and sends its 1-bit decision back to every worker.

The machinery of Theorem \ref{thm:signSGD} is enough to establish convergence for the (good) scheme, but majority vote is more elegant \emph{and} more communication efficient, therefore we focus on this scheme from here on. 

In Theorem \ref{thm:majority} we first establish the general convergence rate of majority vote, followed by a regime where majority vote enjoys a variance reduction from $\|\vec{\sigma}\|_1$ to $\|\vec{\sigma}\|_1/\sqrt{M}$.

\begin{tcolorbox}[boxsep=0pt, arc=0pt,
    boxrule=0.5pt,
 colback=white]
\begin{restatable}[Non-convex convergence rate of distributed
\signSGD{} with majority vote]{theorem}{majoritytheorem}\label{thm:majority}
Run algorithm \ref{alg:majority} for $K$ iterations under Assumptions 1 to 3. Set the learning rate and mini-batch size for each worker (independently of step $k$) as
  \begin{equation*}
  \delta_k = \frac{1}{\sqrt{\Lnorm K}} \qquad \qquad n_k = K
  \end{equation*}
  Then \textbf{(a)} majority vote with $M$ workers converges at least as fast as $\signSGD$ in Theorem \ref{thm:signSGD}.\\
  
  And \textbf{(b)} further assuming that the noise in each component of the stochastic gradient is unimodal and symmetric about the mean (e.g.\ Gaussian), majority vote converges at improved rate:
   \begin{align*}
&\mathbb{E}\sq*{\frac{1}{K}\sum_{k=0}^{K-1} \norm{g_k}_1}^2 \\
    & \qquad \leq \frac{1}{\sqrt{N}}\sq*{\sqrt{\Lnorm}\br*{f_0 - f_*+\frac{1}{2}} + \frac{2}{\sqrt{M}}\|\vec\sigma\|_1}^2
    \end{align*}
      where $N$ is the cumulative number of stochastic gradient calls per worker up to step $K$. 
\end{restatable}
\end{tcolorbox}

The proof is given in the supplementary material, but here we sketch some details. Consider the signal-to-noise ratio of a single component of the stochastic gradient, defined as $S := \frac{|g_i|}{\sigma_i}$. For $S<1$ the gradient is small and it doesn't matter if we get the sign wrong. For $S>1$, we can show using a one-sided version of Chebyshev's inequality \cite{cantelli} that the failure probability, $q$, of that sign bit on an individual worker satisfies $q<\frac{1}{2}$. This means that the parameter server is essentially receiving a repetition code $R_M$ and the majority vote decoder is known to drive down the failure probability of a repetition code exponentially in the number of repeats \cite{MacKay:2002:ITI:971143}.

\textbf{Remark:} Part (a) of the theorem does not describe a speedup over just using a single machine, and that might hint that all those extra $M-1$ workers are a waste in this setting. \textbf{This is not the case}. From the proof sketch above, it should be clear that part (a) is an extremely conservative statement. In particular, we expect all regions of training where the signal-to-noise ratio of the stochastic gradient satisfies $S > 1$ to enjoy a significant speedup due to variance reduction. It's just that since we don't get the speedup when $S < 1$, it's hard to express this in a compact bound.

To sketch a proof for part (b), note that a sign bit from each worker is a Bernoulli trial---call its failure probability $q$. We can get a tight control of $q$ by a convenient tail bound owing to \citet{gauss} that holds under conditions of unimodal symmetry. Then the sum of bits received by the parameter server is a binomial random variable, and we can use Cantelli's inequality to bound its tail. This turns out to be enough to get tight enough control on the error probability of the majority vote decoder to prove the theorem.

\textbf{Remark 1:} assuming that the stochastic gradient of each worker is approximately symmetric and unimodal is very reasonable. In particular for increasing mini-batch size it will be an ever-better approximation by the central limit theorem. Figure \ref{fig:symmetry} plots histograms of real stochastic gradients for neural networks. Even at batch-size 256 the stochastic gradient for an Imagenet model already looks Gaussian.

\textbf{Remark 2:} if you delve into the proof of Theorem \ref{thm:majority} and graph all of the inequalities, you will notice that some of them are uniformly slack. This suggests that the assumptions of symmetry and unimodality can actually be relaxed to only hold approximately. This raises the possibility of proving a relaxed form of Gauss' inequality and using a third moment bound in the Berry-Esseen theorem to derive a minimal batch size for which the majority vote scheme is guaranteed to work by the central limit theorem. We leave this for future work.

\textbf{Remark 3:} why does this theorem have anything to do with unimodality or symmetry at all? It's because there exist very skewed or bimodal random variables $X$ with mean $\mu$ such that $\mathbb{P}[\text{sign}(X) = \text{sign}(\mu)]$ is arbitrarily small. This can either be seen by applying Cantelli's inequality which is known to be tight, or by playing with distributions like
\begin{align*}
	\mathbb{P}[X=x] = \begin{cases}
    0.1 & \text{if } x=50\\
    0.9 & \text{if } x=-1
    \end{cases}
\end{align*}
Distributions like these are a problem because it means that adding more workers will actually drive up the error probability rather than driving it down. The beauty of the central limit theorem is that even for such a skewed and bimodal distribution, the mean of just a few tens of samples will already start to look Gaussian.
\section{Extending the Theory to \Signum{}}

Momentum is a popular trick used by neural network practitioners that can, in our experience, speed up the training of deep neural networks and improve the robustness of algorithms to other hyperparameter settings. Instead of taking steps according to the gradient, momentum algorithms take steps according to a running average of recent gradients.

Existing theoretical analyses of momentum often rely on the absence of gradient stochasticity (e.g.\ \citet{chijin}) or convexity (e.g.\ \citet{goh2017why}) to show that momentum's asymptotic convergence rate can beat gradient descent.

It is easy to incorporate momentum into \signSGD{}, merely by taking the \underline{\smash{sign}} of the moment\underline{\smash{um}} We call the resulting algorithm \Signum{} and present the algorithmic step formally in Algorithm \ref{alg:signum}. \Signum{} fits into our theoretical framework, and we prove its convergence rate in Theorem \ref{thm:signum}.

\begin{tcolorbox}[boxsep=0pt, arc=0pt,
    boxrule=0.5pt,
 colback=white]
\begin{restatable}[Convergence rate of \Signum{}]{theorem}{signumtheorem}\label{thm:signum}
In Algorithm \ref{alg:signum}, set the learning rate, mini-batch size and momentum parameter respectively as
  \begin{equation*}
  \delta_k = \frac{\delta}{\sqrt{k+1}} \qquad \qquad n_k = k+1 \qquad \qquad \beta
  \end{equation*}
 Our analysis also requires a warmup period to let the bias in the momentum settle down. The warmup should last for $C(\beta)$ iterations, where $C$ is a constant that depends on the momentum parameter $\beta$ as follows:
  \begin{align*}
  C(\beta) = \min_{C \in \mathbb{Z}^+} C \quad \text{s.t.}& \quad \frac{C}{2} \beta^C \leq \frac{1}{1-\beta^2} \frac{1}{C+1} \\
  &\quad \text{\&} \quad \beta^{C+1} \leq \frac{1}{2}
  \end{align*}
  Note that for $\beta=0.9$, we have $C=54$ which is negligible. For the first $C(\beta)$ iterations, accumulate the momentum as normal, but use the sign of the stochastic gradient to make updates instead of the sign of the momentum.\\
  
  Let $N$ be the cumulative number of stochastic gradient calls up to step $K$, i.e.\ $N = \text{O}(K^2)$. Then for $K \gg C$ we have
  \begin{align*}
        &\mathbb{E}\sq*{ \frac{1}{K-C}\sum_{k=C}^{K-1} \norm{g_k}_1}^2
    = \text{O}\left(\frac{1}{\sqrt{N}}\left[\frac{f_C - f_*}{\delta}\right.\right. \\
    &\left.\left.+ (1 + \log N) \br*{\frac{\delta\|\vec L\|_1 }{1-\beta} + \|\vec \sigma\|_1\sqrt{1-\beta} }\right]^2\right)
    \end{align*}
where we have used O$(.)$ to hide numerical constants and the $\beta$-dependent constant $C$. 
    \end{restatable}
\end{tcolorbox}

The proof is the greatest technical challenge of the paper, and is given in the supplementary material. We focus on presenting the proof in a modular form, anticipating that parts may be useful in future theoretical work. It involves a very general master lemma, Lemma \ref{lem:master}, which can be used to help prove all the theorems in this paper. 

\textbf{Remark 1:} switching optimisers after a warmup period is in fact commonly done by practitioners \cite{15min}.

\textbf{Remark 2:} the theory suggests that momentum can be used to control a bias-variance tradeoff in the quality of stochastic gradient estimates. Sending $\beta \rightarrow 1$ kills the variance term in $\|\vec \sigma\|_1$ due to averaging gradients over a longer time horizon. But averaging in stale gradients induces bias due to curvature of $f(x)$, and this blows up the $\delta \|\vec L\|_1$ term.

\textbf{Remark 3:} for generality, we state this theorem with a tunable learning rate $\delta$. For variety, we give this theorem in any-time form with a growing batch size and decaying learning rate. This comes at the cost of $\log$ factors appearing.

\begin{figure}[t]
\vskip 0.2in
\begin{center}
  \centerline{\includegraphics[width=\columnwidth]{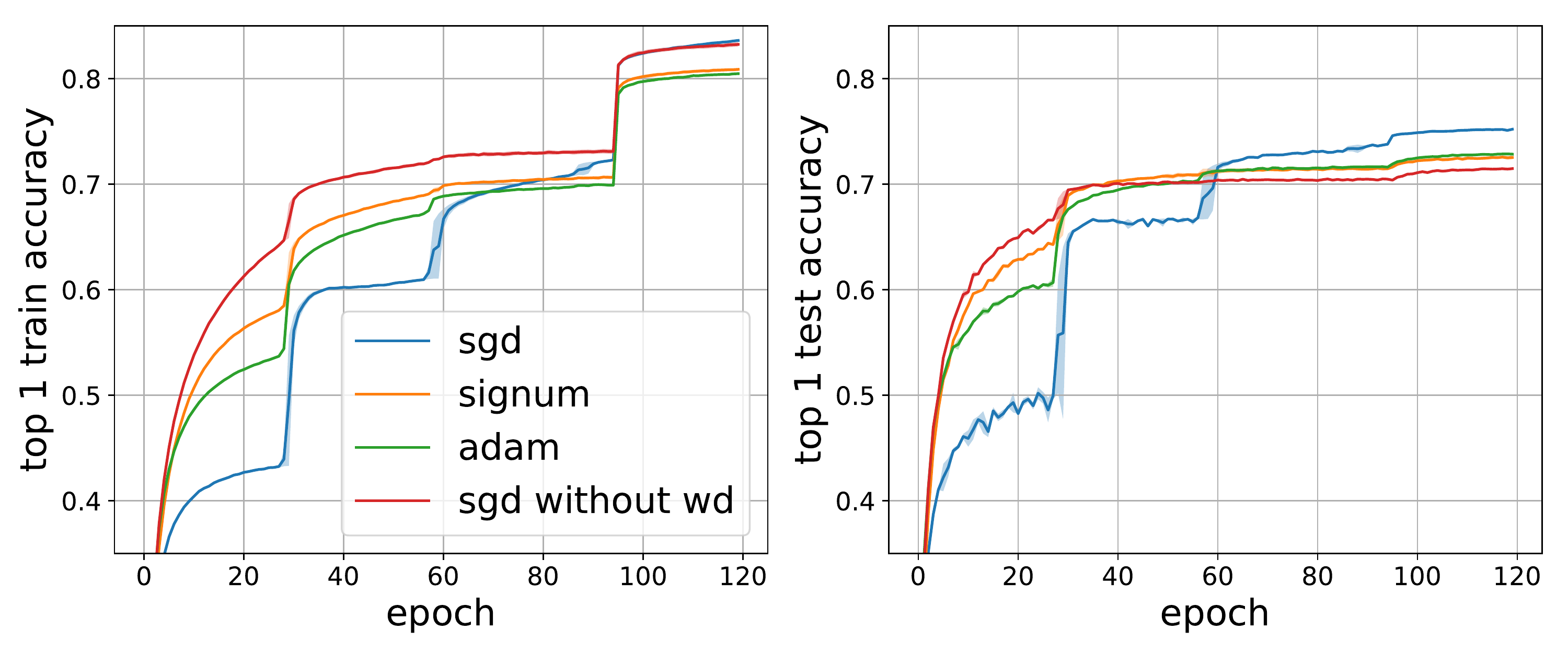}}
\caption{Imagenet train and test accuracies using the momentum version of \signSGD{}, called \Signum{}, to train Resnet-50 v2. We based our implementation on an open source implementation by \url{github.com/tornadomeet}. Initial learning rate and weight decay were tuned on a separate validation set split off from the training set and all other hyperparameters were chosen to be those found favourable for \SGD{} by the community. There is a big jump at epoch 95 when we switch off data augmentation. \Signum{} gets test set performance approximately the same as \Adam{}, better than \SGD{} with out weight decay, but about 2\% worse than \SGD{} with a well-tuned weight decay.
}
\label{fig:imagenet}
\end{center}
\vskip -0.2in
\end{figure}

We benchmark \Signum{} on Imagenet (Figure \ref{fig:imagenet}) and CIFAR-10 (Figure \ref{fig:cifar-test} of supplementary). The full results of a giant hyperparameter grid search for the CIFAR-10 experiments are also given in the supplementary. \Signum{}'s performance rivals \Adam{}'s in all experiments.
\section{Discussion}

Gradient compression schemes like \TernGrad{} \cite{wen2017terngrad} quantise gradients into three levels $\{0,\pm1\}$. This is desirable when the ternary quantisation is sparse, since it can allow further compression. Our scheme of majority vote should easily be compatible with a ternary quantisation---in both directions of communication. This can be
cast as ``majority vote with abstention''. The scheme is as follows: workers send their vote to the parameter server, unless they are very unsure about the sign of the true gradient in which case they send zero. The parameter-server counts the votes, and if quorum is not reached (i.e.\ too many workers disagreed or abstained) the parameter-server sends back zero. This extended algorithm should readily fit into our theory.

In Section \ref{sec:related} we pointed out that \signSGD{} and \Signum{} are closely related to \Adam{}.  In all our experiments we find that \Signum{} and \Adam{} have very similar performance, although both lose out to \SGD{} by about 2\% test accuracy on Imagenet. \citet{marginal}  observed that \Adam{} tends to generalise slightly worse than \SGD{}. Though it is still unclear why this is the case,
perhaps it could be because we don't know how to properly regularise such methods. Whilst we found that neither standard weight decay nor the suggestion of \citet{hutter} completely closed our Imagenet test set gap with \SGD{}, it is possible that some other regularisation scheme might. One idea, suggested by our theory, is that \signSGD{} could be squashing down noise levels. There is some evidence \cite{smithbayes} that a certain level of noise can be good for generalisation, biasing the optimiser towards wider valleys in the objective function. Perhaps, then, adding Gaussian noise to the \Signum{} update might help it generalise better. This can be achieved in a communication efficient manner in the distributed setting by sharing a random seed with each worker, and then generating the same noise on each worker. 


Finally, in Section \ref{sec:theorysection} we discuss some geometric implications of our theory, and provide an efficient and robust experimental means of measuring one aspect---the ratio between noise and gradient density---through the Welford algorithm. We believe that since this density ratio is easy to measure, it may  be useful to help guide those doing architecture search, to find network architectures which are amenable to fast training through gradient compression schemes.
\section{Conclusion}

We have presented a general framework for studying sign-based methods in stochastic non-convex optimisation. We present non-vacuous bounds for gradient compression schemes, and elucidate the special $\ell_1$ geometries under which these schemes can be expected to succeed. Our theoretical framework is broad enough to handle signed-momentum schemes---like \Signum{}---and also multi-worker distributed schemes---like majority vote.

Our work touches upon interesting aspects of the geometry of high-dimensional error surfaces, which we wish to explore in future work. But the next step for us will be to reach out to members of the distributed systems community to help benchmark the majority vote algorithm which shows such great theoretical promise for 1-bit compression in both directions between parameter-server and workers.

\section*{Acknowledgments}

The authors are grateful to the anonymous reviewers for their helpful comments, as well as Jiawei Zhao, Michael Tschannen, Julian Salazar, Tan Nguyen, Fanny Yang, Mu Li, Aston Zhang and Zack Lipton for useful discussions. Thanks to Ryan Tibshirani for pointing out the connection to steepest descent.

KA is supported in part by NSF Career Award CCF-1254106 and Air Force FA9550-15-1-0221. AA is supported in part by Microsoft Faculty Fellowship, Google Faculty Research Award, Adobe Grant, NSF Career Award CCF-1254106, and AFOSR YIP FA9550-15-1-0221.

\bibliography{refs}
\bibliographystyle{icml/icml2018}
\newpage
\onecolumn
\appendix
\renewcommand\dblfloatpagefraction{.9} 
\renewcommand\topfraction{.9}
\renewcommand\dbltopfraction{.9} 
\renewcommand\bottomfraction{.9}
\renewcommand\textfraction{.1}   

\section{Further experimental results}
\begin{figure}[tbh]
\begin{center}
\centerline{\includegraphics[width=0.6\textwidth]{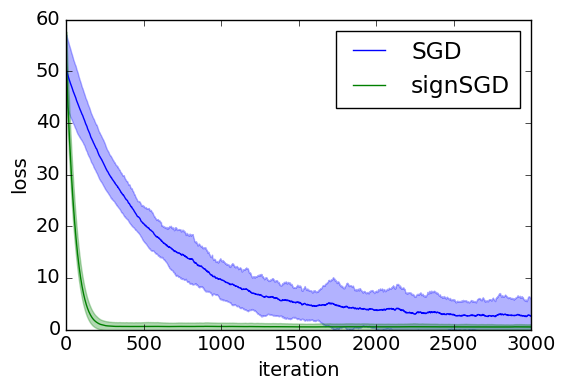}}
\caption{A simple toy problem where \signSGD{} converges faster than \SGD{}. The objective function is just a quadratic $f(x) = \frac{1}{2} x^2$ for $x\in\mathbb{R}^{100}$. The gradient of this function is just $g(x)=x$. We construct an artificial stochastic gradient by adding Gaussian noise $\mathcal{N}(0,100^2)$ to only the first component of the gradient. Therefore the noise is extremely sparse. The initial point is sampled from a unit variance spherical Gaussian. For each algorithm we tune a separate, constant learning rate finding 0.001 best for \SGD{} and 0.01 best for \signSGD{}. \signSGD{} appears more robust to the sparse noise in this problem. Results are averaged over 50 repeats with $\pm 1$ standard deviation shaded.}
\label{fig:sparse_noise}
\end{center}

\vskip 0.2in
\begin{center}
\centerline{\includegraphics[width=0.6\textwidth]{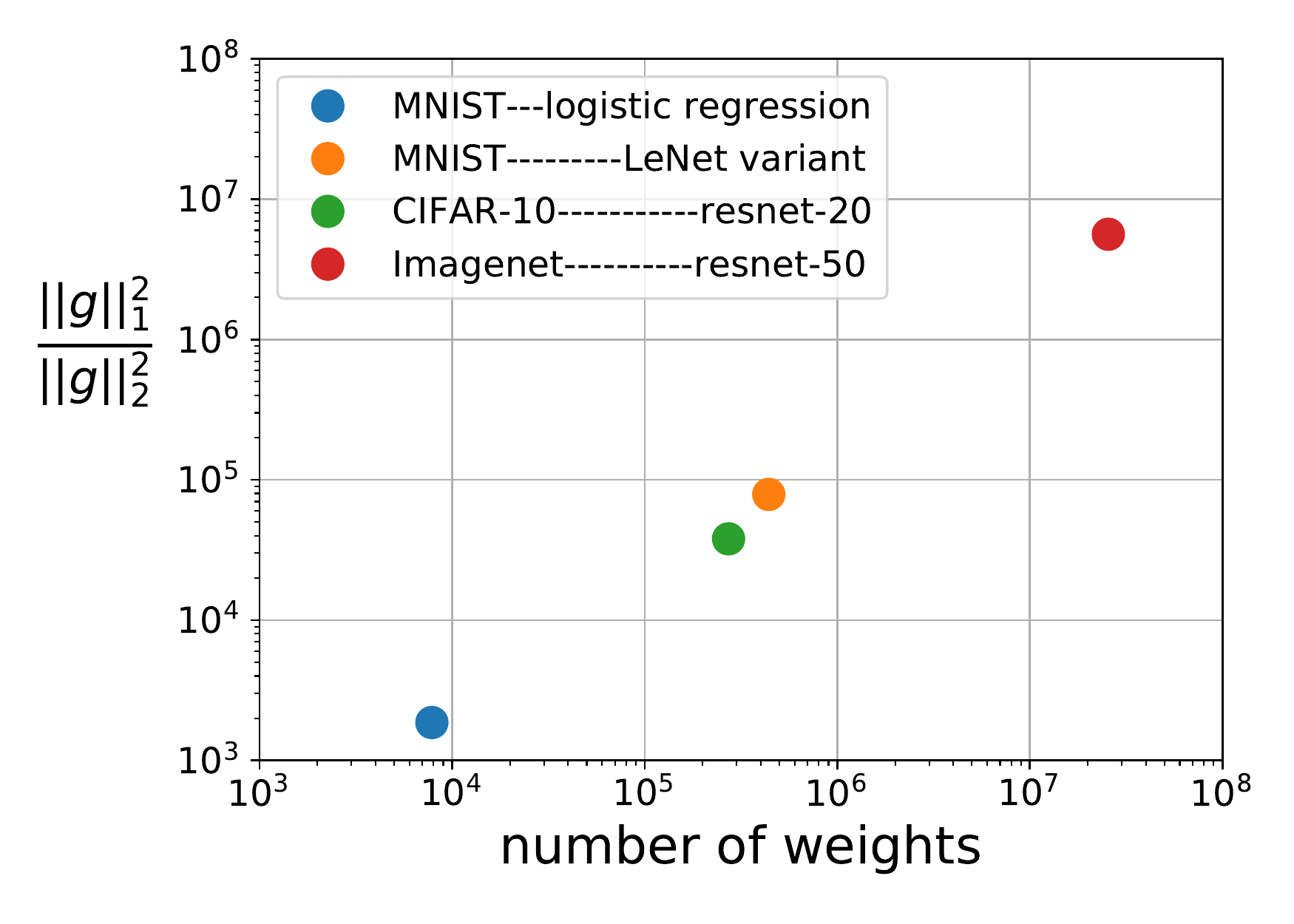}}
\caption{Measuring gradient density via ratio of norms, over a range of datasets and architectures. For each network, we take a point in parameter space provided by the Xavier initialiser \cite{pmlr-v9-glorot10a}. We do a full pass over the data to compute the full gradient at this point. It is remarkably dense in all cases.}
\label{fig:grad-dense}
\end{center}
\vskip -0.2in
\end{figure}

\begin{figure*}[ht]
\vskip 0.2in
\begin{center}
\centerline{\includegraphics[width=0.9\textwidth]{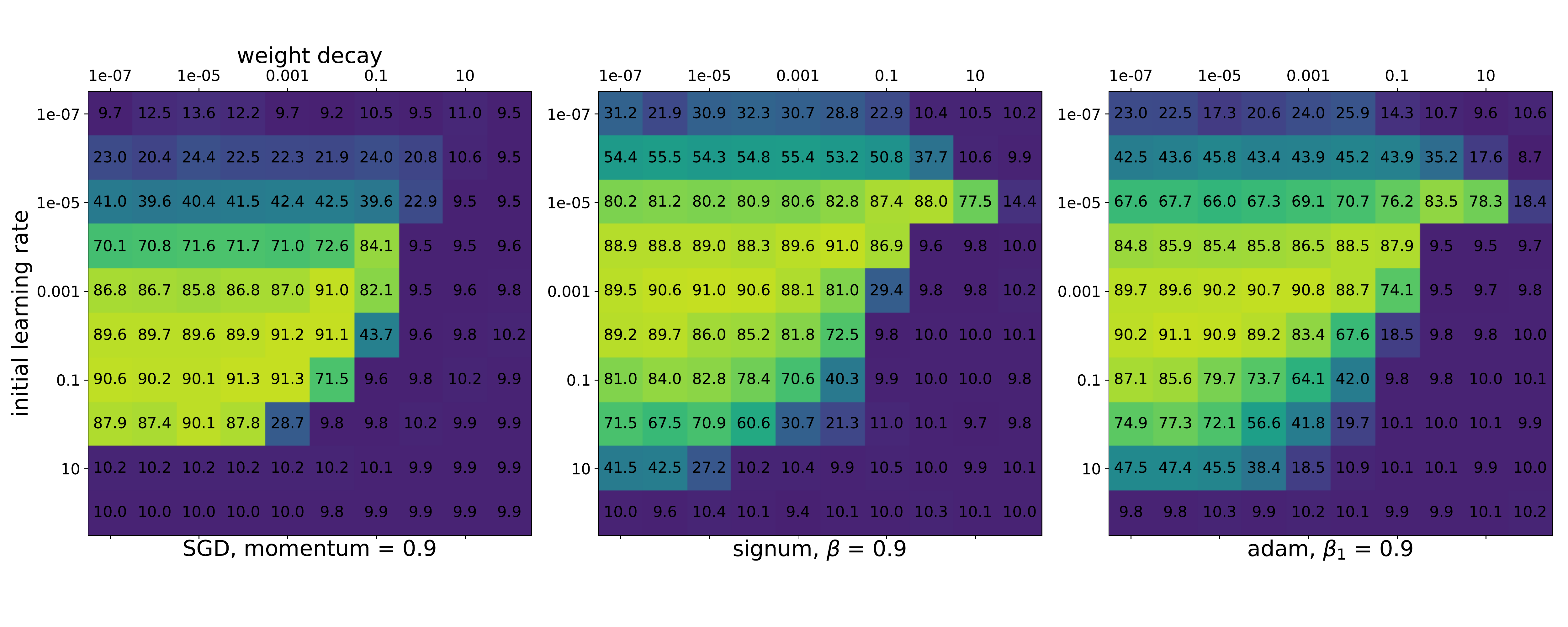}}
\centerline{\includegraphics[width=0.9\textwidth]{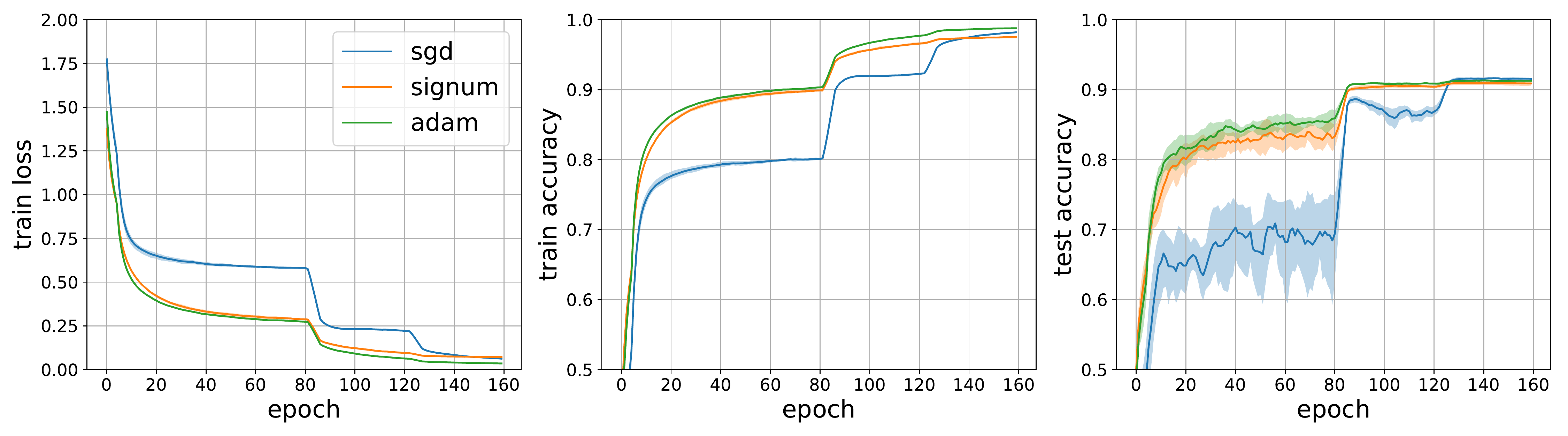}}
\caption{CIFAR-10 results using \Signum{} to train a Resnet-20 model. Top: validation accuracies from a hyperparameter sweep on a separate validation set carved out from the training set. We used this to tune initial learning rate, weight decay and momentum for all algorithms. All other hyperparameter settings were chosen as in \cite{he2016deep} as found favourable for \SGD{}. The hyperparameter sweep for other values of momentum is plotted in Figure \ref{fig:cifar-grid} of the supplementary. Bottom: there is little difference between the final test set performance of the algorithms. \Signum{} closely resembles \Adam{} in all of these plots. 
}
\label{fig:cifar-test}
\end{center}
\vskip -0.2in
\end{figure*}

\begin{figure*}[p]
    \centering
$\overbrace{\qquad\qquad\qquad\qquad\qquad\qquad\qquad\qquad\qquad\qquad\qquad\qquad\qquad\qquad\qquad\qquad\qquad\qquad\qquad}^{\text{\SGD{}}}$

\includegraphics[width=0.85\textwidth]{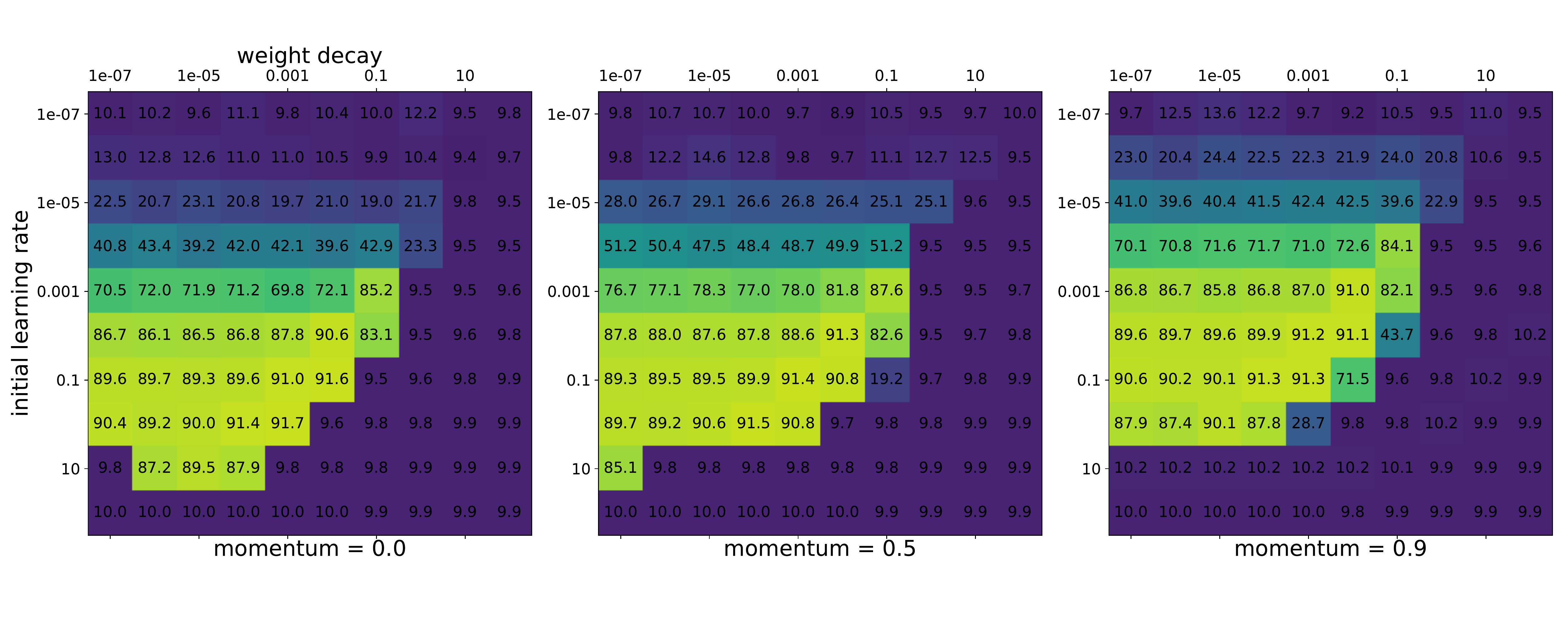}
$\overbrace{\qquad\qquad\qquad\qquad\qquad\qquad\qquad\qquad\qquad\qquad\qquad\qquad\qquad\qquad\qquad\qquad\qquad\qquad\qquad}^{\text{\signSGD{}}}$
 \includegraphics[width=0.85\textwidth]{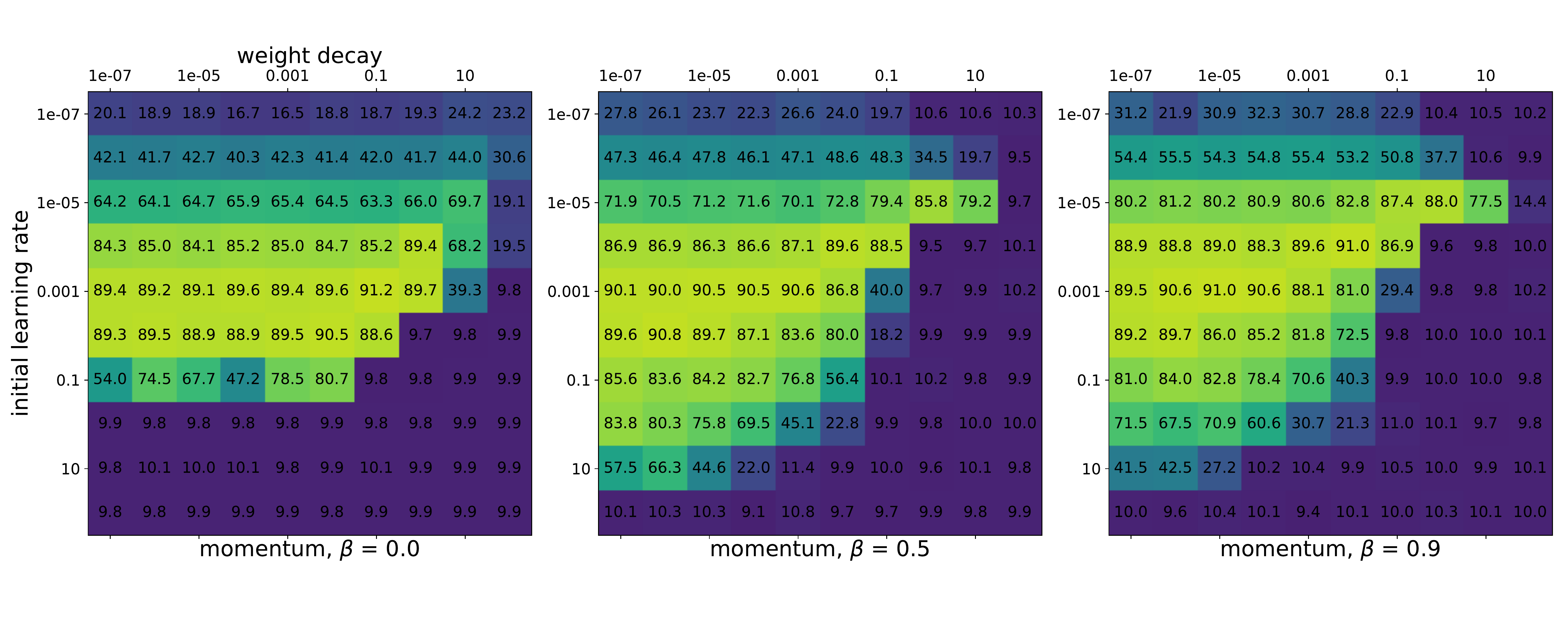}
$\overbrace{\qquad\qquad\qquad\qquad\qquad\qquad\qquad\qquad\qquad\qquad\qquad\qquad\qquad\qquad\qquad\qquad\qquad\qquad\qquad}^{\text{\Adam{}}}$
  \includegraphics[width=0.85\textwidth]{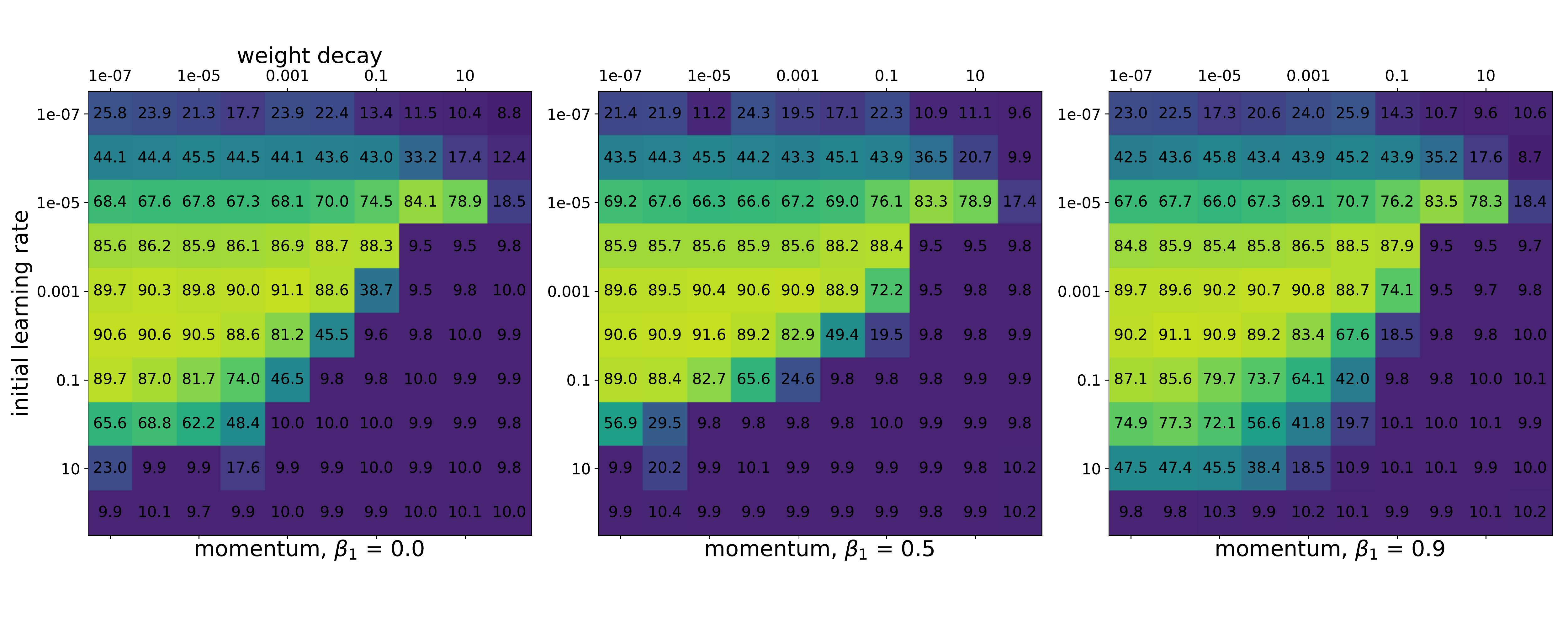}
            
    \caption{Results of a massive grid search over hyperparameters for training Resnet-20 \citep{he2016deep} on CIFAR-10 \citep{Krizhevsky09learningmultiple}. All non-algorithm specific hyperparameters (such as learning rate schedules) were set as in \cite{he2016deep}. In \Adam{}, $\beta_2$ and $\epsilon$ were chosen as recommended in \cite{kingma_adam:_2015}.
Data was divided according to a \{45k/5k/10k\} \{train/val/test\} split. Validation accuracies are plotted above, and the best performer on the validation set was chosen for the final test run (shown in Figure \ref{fig:cifar-test}). All algorithms at the least get close to the baseline reported in \citep{he2016deep} of 91.25\%. Note the broad similarity in general shape of the heatmap between \Adam{} and \signSGD{}, supporting a notion of algorithmic similarity. Also note that whilst \SGD{} has a larger region of very high-scoring hyperparameter configurations, \signSGD{} and \Adam{} appear stable over a larger range of learning rates. Notice that the \SGD{} heat map shifts up for increasing momentum, since the implementation of SGD in the mxnet deep learning framework actually couples the momentum and learning rate parameters.}\label{fig:cifar-grid}
\end{figure*}
\clearpage


\section{Proving the convergence rate of \signSGD{}}
\label{app:sign}

  \begin{tcolorbox}[boxsep=0pt, arc=0pt,
    boxrule=0.5pt,
 colback=white]
\signSGDtheorem*
\end{tcolorbox}

\begin{proof}
First let's bound the improvement of the objective during a single step of the algorithm for one instantiation of the noise. $\mathbb{I}\sq{.}$ is the indicator function, $g_{k,i}$ denotes the $i^{th}$ component of the true gradient $g(x_k)$ and $\tilde{g}_k$ is a stochastic sample obeying Assumption \ref{a:hoeffding}.
  
  First take Assumption \ref{a:coordinate_lip}, plug in the step from Algorithm \ref{alg:signSGD}, and decompose the improvement to expose the stochasticity-induced error:
  \begin{align*}
  f_{k+1}-f_k &\leq g_k^T (x_{k+1}-x_k) + \sum_{i=1}^d \frac{L_i}{2}(x_{k+1} - x_k)_i^2 \\
    &= - \delta_k g_k^T \text{sign}(\tilde{g}_k) + \delta_k^2\sum_{i=1}^d\frac{L_i}{2} \\
    &= - \delta_k \norm{g_k}_1 + \frac{\delta_k^2}{2}\|\vec L\|_1 \\
    & + 2 \delta_k \sum_{i=1}^d |g_{k,i}|\, \mathbb{I}\sq{\text{sign}(\tilde{g}_{k,i}) \neq \text{sign}(g_{k,i})}
  \end{align*}
  
  Next we find the expected improvement at time $k+1$ conditioned on the previous iterate.
  \begin{align*}
  \mathbb{E}\sq{f_{k+1}-f_k|x_k} &\leq - \delta_k \norm{g_k}_1 +\frac{\delta_k^2}{2}\|\vec L\|_1 \\
  & + 2 \delta_k \sum_{i=1}^d |g_{k,i}|\, \mathbb{P}\sq{\text{sign}(\tilde{g}_{k,i}) \neq \text{sign}(g_{k,i})}
  \end{align*}
  
  So the expected improvement crucially depends on the probability that each component of the sign vector is correct, which is intuitively controlled by the relative scale of the gradient to the noise. To make this rigorous, first relax the probability, then use Markov's inequality followed by Jensen's inequality:
  \begin{align*}
  \mathbb{P}\sq{\text{sign}(\tilde{g}_{k,i}) \neq \text{sign}(g_{k,i})} &\leq   \mathbb{P}\sq{\abs{\tilde{g}_{k,i}-g_{k,i}} \geq \abs{g_{k,i}}}\\
	& \leq \frac{\mathbb{E}\sq{\abs{\tilde{g}_{k,i}-g_{k,i}}}}{\abs{g_{k,i}}}\\
	& \leq \frac{\sqrt{\mathbb{E}\sq{(\tilde{g}_{k,i}-g_{k,i})^2}}}{\abs{g_{k,i}}}\\
	&= \frac{\sigma_{k,i}}{\abs{g_{k,i}}}
  \end{align*}
  
  $\sigma_{k,i}$ refers to the variance of the $k^{th}$ stochastic gradient estimate, computed over a mini-batch of size $n_k$. Therefore, by Assumption \ref{a:hoeffding}, we have that $\sigma_{k,i} \leq \sigma_i/\sqrt{n_k}$.
  
    
  We now substitute these results and our learning rate and mini-batch settings into the expected improvement:
  \begin{align*}
  \mathbb{E}\sq{f_{k+1}-f_k|x_k} &\leq - \delta_k \norm{g_k}_1 + 2 \frac{\delta_k}{\sqrt{n_k}} \|\vec\sigma\|_1 +\frac{\delta_k^2}{2}\|\vec L\|_1 \\
  &= - \frac{1}{\sqrt{\Lnorm K}} \norm{g_k}_1 + \frac{2}{\sqrt{\Lnorm} K}\|\vec\sigma\|_1 + \frac{1}{2K}
  \end{align*}
  
  Now extend the expectation over randomness in the trajectory, and perform a telescoping sum over the iterations:
  
  \begin{align*}
    f_0 - f^* &\geq f_0 - \mathbb{E}\sq{f_K} \\
    &= \mathbb{E}\sq*{\sum_{k=0}^{K-1}f_k - f_{k+1}} \\
    &\geq \mathbb{E}\sum_{k=0}^{K-1}\sq*{\frac{1}{\sqrt{\Lnorm K}} \norm{g_k}_1 - \frac{1}{2\sqrt{\Lnorm}K} \br*{4\|\sigma\|_1 + \sqrt{\|\vec L\|_1}} }\\
    &= \sqrt{\frac{K}{\Lnorm}}\mathbb{E}\sq*{\frac{1}{K}\sum_{k=0}^{K-1}\norm{g_k}_1} - \frac{1}{2\sqrt{\Lnorm}} \br*{4\|\vec\sigma\|_1 + \sqrt{\|\vec L\|_1}} \\
  \end{align*}
  
  We can rearrange this inequality to yield the rate:
    
    \begin{align*}
    \mathbb{E}\sq*{\frac{1}{K} \sum_{k=0}^{K-1} \norm{g_k}_1} \leq \frac{1}{\sqrt{K}}\sq*{\sqrt{\Lnorm}\br*{f_0 - f_* + \frac{1}{2}} + 2\|\vec\sigma\|_1}
    \end{align*}
    
    Since we are growing our mini-batch size, it will take $N = O(K^2)$ gradient calls to reach step $K$. Substitute this in, square the result, and we are done.
\end{proof}


\section{Large and small batch \SGD{}}
\label{app:sgdtheory}

\begin{algorithm}[h!]
   \caption{\SGD{}}
   \label{alg:SGD}
\begin{algorithmic}
   \STATE {\bfseries Input:} learning rate $\delta$, current point $x_k$
   \STATE $\tilde{g}_k \leftarrow \mathrm{stochasticGradient}(x_k)$
   \STATE $x_{k+1} \leftarrow x_k - \delta \, \tilde{g}_k$
\end{algorithmic}
\end{algorithm}

For comparison with \signSGD{} theory, here we present non-convex convergence rates for \SGD{}. These are classical results and we are not sure of the earliest reference.

We noticeably get exactly the same rate for large and small batch \SGD{} when measuring convergence in terms of number of stochastic gradient calls. Although the rates are the same for a given number of gradient calls $N$, the large batch setting is preferred (in theory) since it achieves these $N$ gradient calls in only $\sqrt{N}$ iterations, whereas the small batch setting requires $N$ iterations. Fewer iterations in the large batch case implies a smaller wall-clock time to reach a given accuracy (assuming the large batch can be parallelised) as well as fewer rounds of communication in the distributed setting. These systems benefits of large batch learning have been observed by practitioners \cite{Goyal2017AccurateLM}.\\

\begin{tcolorbox}[boxsep=0pt, arc=0pt,
    boxrule=0.5pt,
 colback=white]
\begin{restatable}[Non-convex convergence rate of
\SGD{}]{theorem}{largeSGDtheorem}\label{thm:largeSGD}
Run algorithm \ref{alg:SGD} for $K$ iterations under Assumptions 1 to 3. Define $L := \|L\|_\infty$ and $\sigma^2 := \|\vec \sigma \|_2^2$. Set the learning rate and mini-batch size (independently of step $k$)  as either
  \begin{align*}
  \text{\textbf{(large batch)}} \qquad & \delta_k = \frac{1}{L} & n_k = K \\
  \text{\textbf{(small batch)}} \qquad & \delta_k = \frac{1}{L\sqrt{K}} & n_k = 1
  \end{align*}
  Let   $N$ be the cumulative number of stochastic gradient calls up to step $K$, i.e.\  $N = \text{O}(K^2)$ for large batch and $N = \text{O}(K)$ for small batch. Then, in either case, we have
   \begin{align*}
    &\mathbb{E}\sq*{\frac{1}{K}\sum_{k=0}^{K-1}\norm{g_k}_2^2} \leq \frac{1}{\sqrt{N}}\sq*{2L\br*{f_0 - f_*} + \sigma^2}
    \end{align*}
 
\end{restatable}
\end{tcolorbox}

\begin{proof}
The proof begins the same for the large and small batch case. 

First we bound the improvement of the objective during a single step of the algorithm for one instantiation of the noise. $g_{k}$ denotes the true gradient at step $k$ and $\tilde{g}_k$ is a stochastic sample obeying Assumption \ref{a:hoeffding}.
  
  Take Assumption \ref{a:coordinate_lip} and plug in the algorithmic step.
  \begin{align*}
  f_{k+1}-f_k &\leq g_k^T (x_{k+1}-x_k) + \sum_{i=1}^d \frac{L_i}{2}(x_{k+1} - x_k)_i^2 \\
  &\leq g_k^T (x_{k+1}-x_k) + \frac{\|L\|_\infty}{2} \|x_{k+1} - x_k\|_2^2 \\
    &= - \delta_k g_k^T \tilde{g}_k + \delta_k^2 \frac{L}{2} \|\tilde{g}_k\|_2^2 \\
  \end{align*}
  
  Next we find the expected improvement at time $k+1$ conditioned on the previous iterate.
  \begin{align*}
  \mathbb{E}\sq{f_{k+1}-f_k|x_k} &\leq - \delta_k \norm{g_k}_2^2 +\delta_k^2\frac{L}{2} \left( \sigma_k^2 + \norm{g_k}_2^2\right).
  \end{align*}

  $\sigma_{k}^2$ refers to the variance of the $k^{th}$ stochastic gradient estimate, computed over a mini-batch of size $n_k$. Therefore, by Assumption \ref{a:hoeffding}, we have that $\sigma_k^2 \leq \sigma^2/n_k$. 
  
  First let's substitute in the \textbf{(large batch)} hyperparameters.   \begin{align*}
  \mathbb{E}\sq{f_{k+1}-f_k|x_k} &\leq - \frac{1}{L} \norm{g_k}_2^2 +\frac{1}{2L} \left( \frac{\sigma^2}{K} + \norm{g_k}_2^2\right) \\
  &= - \frac{1}{2L} \norm{g_k}_2^2 +\frac{1}{2L} \frac{\sigma^2}{K}.
  \end{align*}
  
  Now extend the expectation over randomness in the trajectory, and perform a telescoping sum over the iterations:
  
  \begin{align*}
    f_0 - f^* &\geq f_0 - \mathbb{E}\sq{f_K} \\
    &= \mathbb{E}\sq*{\sum_{k=0}^{K-1}f_k - f_{k+1}} \\
    &\geq \frac{1}{2L} \mathbb{E}\sum_{k=0}^{K-1}\sq*{\norm{g_k}_2^2 -\sigma^2}\\
  \end{align*}
  We can rearrange this inequality to yield the rate:
    \begin{align*}
    \mathbb{E}\sq*{\frac{1}{K}\sum_{k=0}^{K-1} \norm{g_k}_2^2} \leq \frac{1}{K}\sq*{2L\br*{f_0 - f_*} + \sigma^2}.
    \end{align*}
    
    Since we are growing our mini-batch size, it will take $N = O(K^2)$ gradient calls to reach step $K$. Substitute this in and we are done for the \textbf{(large batch)} case.

  Now we need to show that the same result holds for the \textbf{(small batch)} case. Following the initial steps of the large batch proof, we get
  \begin{align*}
  \mathbb{E}\sq{f_{k+1}-f_k|x_k} &\leq - \delta_k \norm{g_k}_2^2 +\delta_k^2\frac{L}{2} \left( \sigma_k^2 + \norm{g_k}_2^2\right).
  \end{align*}
  This time $\sigma_{k}^2 = \sigma^2$. Substituting this and our learning rate and mini-batch settings into the expected improvement:
  \begin{align*}
  \mathbb{E}\sq{f_{k+1}-f_k|x_k} &\leq - \frac{1}{L\sqrt{K}} \norm{g_k}_2^2 +\frac{1}{2LK} \left( \sigma^2 + \norm{g_k}_2^2\right) \\
  &\leq - \frac{1}{2L\sqrt{K}} \norm{g_k}_2^2 +\frac{1}{2L} \frac{\sigma^2}{K}.
  \end{align*}
  Now extend the expectation over randomness in the trajectory, and perform a telescoping sum over the iterations:
  \begin{align*}
    f_0 - f^* &\geq f_0 - \mathbb{E}\sq{f_K} \\
    &= \mathbb{E}\sq*{\sum_{k=0}^{K-1}f_k - f_{k+1}} \\
    &\geq \frac{1}{2L} \mathbb{E}\sum_{k=0}^{K-1}\sq*{\frac{\norm{g_k}_2^2}{\sqrt{K}} -\frac{\sigma^2}{K}}.
  \end{align*}
  We can rearrange this inequality to yield the rate:
    \begin{align*}
    \mathbb{E}\sq*{\frac{1}{K}\sum_{k=0}^{K-1} \norm{g_k}_2^2} \leq \frac{1}{\sqrt{K}}\sq*{2L\br*{f_0 - f_*} + \sigma^2}
    \end{align*}
    It will take $N = O(K)$ gradient calls to reach step $K$. Substitute this in and we are done.
\end{proof}




\section{Proving the convergence rate of distributed \signSGD{} with majority vote}
\label{app:majority}

  \begin{tcolorbox}[boxsep=0pt, arc=0pt,
    boxrule=0.5pt,
 colback=white]
\majoritytheorem*
\end{tcolorbox}
Before we introduce the unimodal symmetric assumption, let's first address the claim that M-worker majority vote is at least as good as single-worker \signSGD{} as in Theorem \ref{thm:signSGD} only using Assumptions 1 to 3.
\begin{proof}[Proof of (a)]
Recall that a crucial step in Theorem \ref{thm:signSGD} is showing that
\begin{equation*}
	|g_{i}|\, \mathbb{P}\sq{\text{sign}(\tilde{g}_{i}) \neq \text{sign}(g_{i})} \leq \sigma_{i}
\end{equation*}
for component $i$ of the stochastic gradient with variance bound $\sigma_i$.

The only difference in majority vote is that instead of using $\text{sign}(\tilde{g_i})$ to approximate $\text{sign}(g_i)$, we are instead using $\text{sign}\sq*{\sum_{m=1}^M \text{sign}(\tilde{g}_{m,i})}$. If we can show that the same bound in terms of $\sigma_{i}$ holds instead for \begin{equation*}
	|g_{i}|\, \mathbb{P}\sq*{\text{sign}\sq*{\sum_{m=1}^M \text{sign}(\tilde{g}_{m,i})} \neq \text{sign}(g_{i})} \tag{$\star$}
\end{equation*}
then we are done, since the machinery of Theorem 1 can then be directly applied.

Define the signal-to-noise ratio of a component of the stochastic gradient as $\SNR := |g_i|/\sigma_i$. Note that when $\SNR \leq 1$ then $(\star)$ is trivially satisfied, so we need only consider the case that $\SNR > 1$. $S$ should really be labeled $S_i$ but we abuse notation.

Without loss of generality, assume that $g_{i}$ is negative, and thus using Assumption \ref{a:hoeffding} and Cantelli's inequality \cite{cantelli} we get that for the failure probability $q$ of a single worker

\begin{equation*}
q:=\mathbb{P}[\text{sign}(\tilde{g}_i)\neq\text{sign}(g_i)] = \mathbb{P}[\tilde{g}_i - g_i\geq|g_i|] \leq \frac{1}{1 + \frac{g_i^2}{\sigma_i^2}}
\end{equation*}

For $\SNR > 1$ then we have failure probability $q<\frac{1}{2}$. If the failure probability of a single worker is smaller than $\frac{1}{2}$ then the server is essentially receiving a repetition code $R_M$ of the true gradient sign. Majority vote is the maximum likelihood decoder of the repetition code, and of course decreases the probability of error---see e.g.\ \cite{MacKay:2002:ITI:971143}. Therefore in all regimes of $\SNR$ we have that
\begin{align*}
(\star) \leq |g_{i}|\, \mathbb{P}\sq{\text{sign}(\tilde{g}_{i}) \neq \text{sign}(g_{i})} \leq \sigma_{i}
\end{align*}
and we are done.
\end{proof}

That's all well and good, but what we'd really like to show is that using $M$ workers provides a speedup by reducing the variance. Is
\begin{align*}
(\star) \stackrel{?}{\leq} \frac{\sigma_i}{\sqrt{M}} \tag{$\dagger$}\end{align*}
too much to hope for?

Well in the regime where $\SNR \gg 1$ such a speedup is very reasonable since $q \ll \frac{1}{2}$ by Cantelli, and the repetition code actually supplies exponential reduction in failure rate. But we need to exclude very skewed or bimodal distributions where $q>\frac{1}{2}$ and adding more voting workers will not help. That brings us naturally to the following lemma:

\begin{lemma}[Failure probability of a sign bit under conditions of unimodal symmetric gradient noise]\label{lem:symm}

Let $\tilde{g}_i$ be an unbiased stochastic approximation to gradient component $g_i$, with variance bounded by $\sigma_i^2$. Further assume that the noise distribution is unimodal and symmetric. Define signal-to-noise ratio $S:= \frac{|g_i|}{\sigma_i}$. Then we have that
\begin{align*}
\mathbb{P}[\text{sign}(\tilde{g}_i)\neq\text{sign}(g_i)] 
&\leq \begin{cases}
\frac{2}{9}\frac{1}{S^2} & \quad \text{if } S > \frac{2}{\sqrt{3}},\\
\frac{1}{2}-\frac{S}{2\sqrt{3}} & \quad \text{otherwise}
\end{cases}
\end{align*}
which is in all cases less than $\frac{1}{2}$.
\end{lemma}
\begin{proof}
Recall Gauss' inequality for unimodal random variable X with mode $\nu$ and expected squared deviation from the mode $\tau^2$ \cite{gauss,threesigma}:
\begin{equation*}
\mathbb{P}[|X-\nu| > k] \leq \begin{cases}
\frac{4}{9}\frac{\tau^2}{k^2} & \quad \text{if } \frac{k}{\tau}> \frac{2}{\sqrt{3}},\\
1-\frac{k}{\sqrt{3}\tau} & \quad \text{otherwise}
\end{cases}
\end{equation*}
By the symmetry assumption, the mode is equal to the mean, so we replace mean $\mu = \nu$ and variance $\sigma^2 = \tau^2$.
\begin{equation*}
\mathbb{P}[|X-\mu| > k] \leq \begin{cases}
\frac{4}{9}\frac{\sigma^2}{k^2} & \quad \text{if } \frac{k}{\sigma}> \frac{2}{\sqrt{3}},\\
1-\frac{k}{\sqrt{3}\sigma} & \quad \text{otherwise}
\end{cases}
\end{equation*}
Without loss of generality assume that $g_i$ is negative. Then applying symmetry followed by Gauss, the failure probability for the sign bit satisfies:
\begin{align*}
\mathbb{P}[\text{sign}(\tilde{g}_i)\neq\text{sign}(g_i)] &= \mathbb{P}[\tilde{g}_i - g_i\geq|g_i|] \\
&=\frac{1}{2}\mathbb{P}[|\tilde{g}_i - g_i|\geq|g_i|] \\
&\leq \begin{cases}
\frac{2}{9}\frac{\sigma_i^2}{g_i^2} & \quad \text{if } \frac{|g_i|}{\sigma}> \frac{2}{\sqrt{3}},\\
\frac{1}{2}-\frac{|g_i|}{2\sqrt{3}\sigma_i} & \quad \text{otherwise}
\end{cases}\\
&= \begin{cases}
\frac{2}{9}\frac{1}{S^2} & \quad \text{if } S > \frac{2}{\sqrt{3}},\\
\frac{1}{2}-\frac{S}{2\sqrt{3}} & \quad \text{otherwise}
\end{cases}
\end{align*}
\end{proof}

We now have everything we need to prove part (b) of Theorem \ref{thm:majority}.

\begin{proof}[Proof of (b)]
If we can show $(\dagger)$ we'll be done, since the machinery of Theorem \ref{thm:signSGD} follows through with $\sigma$ replaced everywhere by $\frac{\sigma}{\sqrt{M}}$. Note that the important quantity appearing in $(\star)$ is 
\begin{equation*}
\sum_{m=1}^M \text{sign}(\tilde{g}_{m,i}).
\end{equation*}
Let $Z$ count the number of workers with correct sign bit. To ensure that 
\begin{equation*}
	\text{sign}\sq*{\sum_{m=1}^M \text{sign}(\tilde{g}_{m,i})} = \text{sign}(g_{i})
\end{equation*}
$Z$ must be larger than $\frac{M}{2}$. But $Z$ is the sum of $M$ independent Bernoulli trials, and is therefore binomial with success probability $p$ and failure probability $q$ to be determined. Therefore we have reduced proving $(\dagger)$ to showing that
\begin{equation*}
	\mathbb{P}\sq*{Z \leq \frac{M}{2}} \leq \frac{1}{\sqrt{M}S} \tag{$\mathsection$}
\end{equation*}
where $Z$ is the number of successes of a binomial random variable $b(M,p)$ and $S$ is our signal-to-noise ratio $S:=\frac{|g_i|}{\sigma_i}$.

Let's start by getting a bound on the success probability $p$ (or equivalently failure probability $q$) of a single Bernoulli trial. 

By Lemma \ref{lem:symm}, which critically relies on unimodal symmetric gradient noise, the failure probability for the sign bit of a single worker satisfies:
\begin{align*}
q&:=\mathbb{P}[\text{sign}(\tilde{g}_i)\neq\text{sign}(g_i)] \\
&\leq \begin{cases}
\frac{2}{9}\frac{1}{S^2} & \quad \text{if } S > \frac{2}{\sqrt{3}},\\
\frac{1}{2}-\frac{S}{2\sqrt{3}} & \quad \text{otherwise}
\end{cases}\\
&:= \tilde{q}(S)
\end{align*}
Where we have defined $\tilde{q}(S)$ to be our $S$-dependent bound on $q$. Since $q \leq \tilde{q}(S) < \frac{1}{2}$, there is hope to show $(\dagger)$. Define $\epsilon$ to be the defect of $q$ from one half, and let $\tilde{\epsilon}(S)$ be its $S$-dependent bound.
\begin{equation*}
\epsilon := \frac{1}{2} - q = p - \frac{1}{2} \geq \frac{1}{2} - \tilde{q}(S) := \tilde{\epsilon}(S)
\end{equation*}

Now we have an analytical handle on random variable $Z$, we may proceed to show $(\mathsection)$. There are a number of different inequalities that we can use to bound the tail of a binomial random variable, but Cantelli's inequality will be good enough for our purposes. 

Let $\bar{Z} := M - Z$ denote the number of failures. $\bar{Z}$ is binomial with mean $\mu_{\bar{Z}} = Mq$ and variance $\sigma^2_{\bar{Z}} = M p q$. Then using Cantelli we get
\begin{align*}
	\mathbb{P}\sq*{Z \leq \frac{M}{2}} &= 	\mathbb{P}\sq*{\bar{Z} \geq \frac{M}{2}}\\
    &= 	\mathbb{P}\sq*{\bar{Z}-\mu_{\bar{Z}} \geq \frac{M}{2}-\mu_{\bar{Z}}}\\
    &= 	\mathbb{P}\sq*{\bar{Z}-\mu_{\bar{Z}} \geq M \epsilon}\\ 
    &\leq \frac{1}{1+ \frac{M^2\epsilon^2}{Mpq}}\\
    &\leq \frac{1}{1+ \frac{M}{\frac{1}{4\epsilon^2}-1}}
\end{align*}
Now using the fact that $\frac{1}{1+x^2} \leq \frac{1}{2x}$ we get
\begin{align*}
	\mathbb{P}\sq*{Z \leq \frac{M}{2}} &\leq \frac{\sqrt{\frac{1}{4\epsilon^2}-1}}{2\sqrt{M}}
\end{align*}

To finish, we need only show that $\sqrt{\frac{1}{4\epsilon^2}-1}$ is smaller than $\frac{2}{S}$, or equivalently that its square is smaller than $\frac{4}{S^2}$. Well plugging in our bound on $\epsilon$ we get that
\begin{align*}
\frac{1}{4\epsilon^2}-1 &\leq \frac{1}{4\tilde{\epsilon}(S)^2}-1 
\end{align*}
where
\begin{equation*}
\tilde{\epsilon}(S) = \begin{cases}
\frac{1}{2} - \frac{2}{9}\frac{1}{S^2} & \quad \text{if } S > \frac{2}{\sqrt{3}},\\
\frac{S}{2\sqrt{3}} & \quad \text{otherwise}
\end{cases}
\end{equation*}
First take the case $S \leq \frac{2}{\sqrt{3}}$. Then $\tilde{\epsilon}^2 = \frac{S^2}{12}$ and $\frac{1}{4\tilde{\epsilon}^2}-1 = \frac{3}{S^2} - 1 < \frac{4}{S^2}$. Now take the case $S > \frac{2}{\sqrt{3}}$. Then $\tilde{\epsilon} = \frac{1}{2} - \frac{2}{9}\frac{1}{S^2}$ and we have $\frac{1}{4\tilde{\epsilon}^2}-1 = \frac{1}{S^2} \frac{\frac{8}{9} - \frac{16}{81}\frac{1}{S^2}}{1 -\frac{8}{9}\frac{1}{S^2} + \frac{16}{81}\frac{1}{S^4}} < \frac{1}{S^2}\frac{\frac{8}{9}}{1-\frac{8}{9}\frac{1}{S^2}}<\frac{4}{S^2}$ by the condition on $S$.

So we have shown both cases, which proves $(\mathsection)$ from which we get $(\dagger)$ and we are done.
\end{proof}

\section{General recipes for the convergence of approximate sign gradient methods}
\label{app:proof}

Now we generalize the arguments in the proof of \signSGD{} and prove a master lemma that provides a general recipe for analyzing the approximate sign gradient method. This allows us to handle momentum and the majority voting schemes, hence proving Theorem~\ref{thm:signum} and Theorem~\ref{thm:majority}.


\begin{lemma}[Convergence rate for a class of approximate sign gradient method]\label{lem:master}
Let $C$ and $K$ be integers satisfying $0 < C \ll K$. Consider the algorithm given by $x_{k+1} = x_k - \delta_k \text{sign}(v_k)$,
	for a fixed positive sequence of $\delta_k$ and
	where $v_k \in \R^d$ is a measurable and square integrable function of the entire history up to time $k$, including $x_1,...,x_k,  v_1,...,v_{k-1}$ and all $N_k$ stochastic gradient oracle calls up to time $k$.
	Let $g_k =  \nabla f(x_k)$.
	If  Assumption~\ref{a:lower} and Assumption~\ref{a:coordinate_lip} are true and in addition for $k=C,C+1,C+2,...,K$
    \begin{equation}\label{eq:masterbound_new}
    \E\left[\sum_{i=1}^d|g_{k,i}|\P[\text{sign}(v_{k,i}) \neq \text{sign}(g_{k,i})| x_k ]\right] \leq \xi(k) 
    \end{equation}
	where the expectation is taken over the all random variables, and the rate $\xi(k)$ obeys that $\xi(k) \rightarrow 0$ as $k\rightarrow \infty$ and  then we have
	$$
	 \frac{1}{K-C}\sum_{k=C}^{K-1}\E\|g_k\|_1  \leq  \frac{f_C  - f_*  + 2\sum_{k=C}^{K-1}\delta_k \xi(k)  +  \sum_{k=C}^{K-1}\frac{\delta_k^2\|\vec L\|_1}{2}}{(K-C)\min_{C\leq k \leq K-1}\delta_k }.
	$$
	In particular, if $\delta_k = \delta/\sqrt{k}$ and $\xi(k)  = \kappa/\sqrt{k}$, for some problem dependent constant $\kappa$, then we have 
	$$\frac{1}{K-C}\sum_{k=C}^{K-1} \E\|g_k\|_1  \leq \frac{\frac{f_C - f_*}{\delta}  + (2\kappa+\|\vec L\|_1\delta/2)(\log K+1) }{\sqrt{K} - \frac{C}{\sqrt{K}}}. $$
\end{lemma}
\begin{proof}
	  Our general strategy will be to show that the expected objective improvement at each step will be good enough to guarantee a convergence rate in expectation. First let's bound the improvement of the objective during a single step of the algorithm for $k\geq C$, and then take expectation. Note that $\mathbb{I}\sq{.}$ is the indicator function, and $w_k[i]$ denotes the $i^{th}$ component of the vector $w_k$.
	
	  By Assumption~\ref{a:coordinate_lip}
	\begin{align*}
	f_{k+1}-f_k &\leq g_k^T (x_{k+1}-x_k) + \sum_{i=1}^d\frac{L_i}{2} (x_{k+1,i}-x_{k,i})^2 &\text{Assumption \ref{a:coordinate_lip}} \\
	&= - \delta_k g_k^T \text{sign}(v_k) + \delta_k^2\frac{\|\vec L\|_1}{2} &\text{by the update rule} \\
	&= - \delta_k \norm{g_k}_1 + 2 \delta_k \sum_{i=1}^d |g_k[i]|\, \mathbb{I}\sq{\text{sign}(v_k[i]) \neq \text{sign}(g_k[i])} + \delta_k^2\frac{\|\vec L\|_1}{2} & \text{by identity}
	\end{align*}
	
	Now, for $k\geq C$ we need to find the expected improvement at time $k+1$ conditioned on $x_k$, where the expectation is over the randomness of the stochastic gradient oracle. Note that $\mathbb{P}\sq{E}$ denotes the probability of event $E$.
		\begin{align*}
	\mathbb{E}\sq{f_{k+1}-f_k|x_k} &\leq - \delta_k \norm{g_k}_1 + 2 \delta_k \sum_{i=1}^d |g_k[i]|\, \mathbb{P}\sq[\Big]{\text{sign}(v_k[i]) \neq \text{sign}(g_k[i]) \Big| x_k} + \delta_k^2\frac{\|\vec L\|_1}{2}.
	\end{align*}
    Note that $g_k$ becomes fixed when we condition on $x_k$. Further take expectation over $x_k$, and apply \eqref{eq:masterbound_new}. We get:
    	\begin{equation}\label{eq:desent}
	\mathbb{E}\sq{f_{k+1}-f_k}  \leq -\delta_k \E[\norm{g_k}_1] + 2\delta_k \xi(k) + \frac{\delta_k^2 \|\vec L\|_1}{2}.
	\end{equation}


	Rearrange the terms and sum over \eqref{eq:desent} for $k=C,C+1,...,K-1$.
	$$
\sum_{k=C}^{K-1} \delta_k \E[\|g_k\|_1]  \leq   \sum_{k=C}^{K-1}(\E f_{k} - \E f_{k+1})  +  2\sum_{k=C}^{K-1}\delta_k\xi(k)  + \sum_{k=C}^{K-1}\frac{\delta_k^2\|\vec L\|_1}{2} 
	$$
	
	Dividing both sides by $\big[(K-C)\min_{C\leq k\leq K-1} \delta_k \big]$, using a telescoping sum over $\E f_k$ and using that $f(x) \geq  f_*$ for all $x$, we get
	 $$\frac{1}{K-C}\sum_{k=C}^{K-1} \E\|g_k\|_1  \leq  \frac{f_C  - f_*  + 2\sum_{k=C}^{K-1}\delta_k \xi(k)  +  \sum_{k=C}^{K-1}\frac{\delta_k^2\|\vec L\|_1}{2}}{(K-C)\min_{C\leq k \leq K-1}\delta_k }$$
	 and the proof is complete by noting that the minimum is smaller than the average in the LHS.
\end{proof}

To use the above Lemma for analyzing \Signum{} and the Majority Voting scheme, it suffices to check condition \eqref{eq:masterbound_new} for each algorithm.

One possible way to establish \eqref{eq:masterbound_new} is show that $v_k$ is a good approximation of the gradient $g_k$ in expected absolute value.
\begin{lemma}[Estimation to testing reduction]
Equation \eqref{eq:masterbound_new} is true, if for every $k$
	\begin{equation} \label{eq:masterbound}
	\sum_{i=1}^d{\E |v_k[i]  - g_k[i]|} \leq \xi(k).
	\end{equation}
\end{lemma}
\begin{proof}
First note that for any two random variables $a,b \in \R$.
$$
\mathbb{P}\sq[\Big]{\text{sign}(a) \neq \text{sign}(b)}\leq \mathbb{P}\sq[\Big]{|a-b|>|b|} .
$$
Condition on $x_k$ and apply the above inequality to every $i=1,...,d$ to what is inside the expectation of \eqref{eq:masterbound_new}, we have
$$\sum_{i=1}^d |g_k[i]| \mathbb{P}\sq[\Big]{\text{sign}(v_k[i]) \neq \text{sign}(g_k[i]) \Big| x_k} \leq \sum_{i=1}^d |g_k[i]| \mathbb{P}\sq[\Big]{|v_k[i]-g_k[i]|> |g_k[i]| \Big| x_k} \leq \sum_{i=1}^d\E[|v_k[i]-g_k[i]| | x_k].$$
Note that the final $\leq$ uses Markov's inequality and constant $|g_k[i]|$ cancels out.
 
The proof is complete by taking expectation on both sides and apply \eqref{eq:masterbound}.
\end{proof}
Note that the proof of this lemma uses Markov's inequality in the same way information-theoretical lower bounds are often proved in statistics --- reducing estimation to testing.

Another handy feature of the result is that we do not require the approximation to hold for every possible $x_k$. It is okay that for some $x_k$, the approximation is much worse as long as those $x_k$ appears with small probability according to the algorithm. This feature enables us to analyze momentum and hence proving the convergence for \Signum.

\section{Analysis for \Signum}

Recall our definition of the key random variables used in \Signum.
\begin{align*}
g_k &:= \nabla f(x_k)\\
\tilde{g}_k &:= \frac{1}{n_k}\sum_{j=1}^{n_k} \tilde{g}^{(j)}(x_k)\\
m_k &:= \frac{1-\beta}{1-\beta^{k+1}}\sum_{t=0}^k \sq*{\beta^t g_{k-t}}\\
\tilde{m}_k &:= \frac{1-\beta}{1-\beta^{k+1}}\sum_{t=0}^k \sq*{\beta^t \tilde{g}_{k-t}}\\
\end{align*}
\Signum{} effectively uses $v_k = \tilde{m}_k$ and also $\delta_k  = O(1/\sqrt{k})$. 

Before we prove the convergence of \Signum, we first prove a utility lemma about the random variable $Z_k   := \tilde{g}_k  -  g_k $. Note that in this lemma quantities like $Z_k$, $Y_k$, $\abs{Z_k}$ and $Z_k^2$ are considered vectors---so this lemma is a statement about each component of the vectors separately and all operations, such as $(\cdot)^2$ are pointwise operations.

\begin{lemma}[Cumulative error of stochastic gradient]\label{lem:martingale}
	For any $k<\infty$ and fixed weight $-\infty < \alpha_1,...,\alpha_k< \infty$, $\sum_{l=1}^k  \alpha_l  Z_l$ is a Martingale. In particular,
	$$
	\E\sq*{\br*{\sum_{l=1}^k  \alpha_l  Z_l}^2} \leq  \sum_{l=1}^k \alpha_l^2 {\vec \sigma}^2.
	$$
\end{lemma}
\begin{proof}
	We simply check the definition of a Martingale. Denote $Y_k:= \sum_{l=1}^k\alpha_l Z_l$. 
	First, we have that
	\begin{align*}
	\E[|Y_k|] &= \E\sq*{\abs*{\sum_{l=1}^k\alpha_l Z_l }} \\  
    &\leq \sum_l  |\alpha_l|\E[|Z_l|] &\text{triangle inequality}\\
    &= \sum_l  |\alpha_l|\E\sq[\Big]{\E\sq*{|Z_l| | x_l}} &\text{law of total probability}\\
   &\leq \sum_l  |\alpha_l|\E\sq*{\sqrt{\E[ Z_l^2  |x_l]} } &\text{Jensen's inequality}\\
   &\leq \sum_l |\alpha_l| \vec \sigma  <\infty 
	\end{align*}
	Second, again using the law of total probability,
	\begin{align*}
	\E[Y_{k+1} | Y_1,...,Y_k]  &= \E\left[\sum_{l=1}^{k+1}\alpha_l Z_l  \middle|  \alpha_1 Z_1, ...,  \alpha_k Z_k \right]  \\
	&=  Y_k  +  \alpha_{k+1}\E\left[ Z_{k+1}  \middle|  \alpha_1 Z_1, ...,  \alpha_k Z_k  \right]\\
	&= Y_k  +   \alpha_{k+1}\E\left[  \E\left[ Z_{k+1}  \middle|  x_{k+1},  \alpha_1 Z_1, ...,  \alpha_k Z_k  \right] | \alpha_1 Z_1, ...,  \alpha_k Z_k \right]\\
	&= Y_k  + \alpha_{k+1}\E\left[  \E\left[ Z_{k+1}  \middle|  x_{k+1}\right] | \alpha_1 Z_1, ...,  \alpha_k Z_k \right]\\
	&=Y_k
	\end{align*}
	This completes the proof that it is indeed a Martingale. We now make use of the properties of Martingale difference sequences to establish a variance bound on the Martingale.
	\begin{align*}
	\E\sq*{\br*{\sum_{l=1}^k \alpha_l Z_l }^2 }  &=  \sum_{l=1}^k \E[  \alpha_l^2 Z_l^2]   +  2\sum_{l<j} \E[\alpha_l\alpha_j Z_l Z_j]\\
	&=\sum_{l=1}^k \alpha_l^2  \E[\E[Z_l^2| Z_1,...,Z_{l-1} ]]   +   2\sum_{l<j} \alpha_l\alpha_j  \E\Big[Z_l \E\big[ \E[Z_j | Z_1,...,Z_{j-1}] \big| Z_l \big]\Big]\\
    &=\sum_{l=1}^k \alpha_l^2  \E[\E[\E[Z_l^2| x_l,Z_1,...,Z_{l-1} ]| Z_1,...,Z_{l-1} ]]   +   0\\
	&=\sum_{l=1}^k \alpha_l^2 \vec \sigma^2.
	\end{align*}
\end{proof}
The consequence of this lemma is that we are able to treat $Z_1,...,Z_k$ as if they are  independent, even though they are not---clearly $Z_l$ is dependent on $Z_1,...,Z_{l-1}$ through $x_l$.\\

\begin{lemma}[Gradient approximation in \Signum{}]\label{lem:gradient_approx}
	The version of the \Signum{} algorithm that takes $v_k = \tilde{m}_k$, and all parameters according to Theorem~\ref{thm:signum}, obeys that for all integer  $C\leq k\leq K$
	$$
	\left\|\E |\tilde{m}_k  - g_k|\right\|_1 \leq \frac{2}{\sqrt{k+1}} \left(8 \|\vec L\|_1 \delta \frac{\beta}{1-\beta}  + \sqrt{3}\|\vec \sigma\|_1\sqrt{1-\beta}\right).
	$$
\end{lemma}
\begin{proof}
For each $i\in[d]$ we will use the following non-standard ``bias-variance'' decomposition.
  \begin{equation}\label{eq:bias-var-decomp}
\E\left[ |\tilde{m}_k[i] - g_k[i]|  \right] \leq  \underbrace{\E\left[ |m_k[i] - g_k[i]|  \right]}_{(*)} +  \underbrace{\E\left[ |\tilde{m}_k[i] - m_k[i]|   \right]}_{(**)}
\end{equation}
We will first bound $(**)$ and then deal with $(*)$.

Note that  $(**)  = \frac{1-\beta}{1-\beta^{k+1}}\E\left[ |\sum_{t=0}^{k} \beta^{k-t}  Z_t|   \right]$.
Using Jensen's inequality and applying Lemma~\ref{lem:martingale} with our choice of $\alpha_1,...,\alpha_l$ (including the effect of the increasing batch size) we get that for $k \geq C$
$$
(**)  \leq \frac{1-\beta}{1-\beta^{k+1}}\sqrt{\E \sq*{\abs*{\sum_{t=0}^{k} \beta^{t} Z_{k-t}}^2 }}  \leq \frac{1-\beta}{1-\beta^{k+1}}\sqrt{  \sum_{t=0}^k \sq*{\beta^{2t} \frac{\vec \sigma^2}{k-t+1}} },
$$
where
\begin{align*}
\sum_{t=0}^k \sq*{\beta^{2t} \frac{\vec \sigma^2}{k-t+1}}  &= 
\sum_{t=0}^{\frac{k}{2}}\sq*{\beta^{2t} \frac{\vec\sigma^2}{k-t+1}}  + \sum_{t=\frac{k}{2}+1}^k \sq*{\beta^{2t} \frac{\vec\sigma^2}{k-t+1}} & \text{break up sum}\\
&\leq 
\sum_{t=0}^{\frac{k}{2}}\sq*{\beta^{2t}} \frac{\vec\sigma^2}{\frac{k}{2}+1} + \sum_{t=\frac{k}{2}+1}^k \sq*{\beta^k \vec\sigma^2} &\text{bound summands}\\
&\leq \frac{1}{1-\beta^2}\frac{\vec\sigma^2}{\frac{k}{2}+1} + \frac{k}{2}\beta^{k}\vec\sigma^2 &\text{geometric series} \\
&\leq \frac{3}{1-\beta^2}\frac{\vec\sigma^2}{k + 1} & \text{since }k \geq C
\end{align*}

Combining, and again using our condition that $k \geq C$, we get
\begin{equation} \label{eq:std_bound}
(**) \leq \frac{1-\beta}{1-\beta^{k+1}}\sqrt{\frac{3}{1-\beta^2}}\frac{\vec\sigma}{\sqrt{k+1}} \leq 2\sqrt{3}\sqrt{1-\beta}\frac{\vec\sigma}{\sqrt{k+1}}
\end{equation}

We now turn to bounding $(*)$ --- the ``bias'' term.
\begin{align}
\mathbb{E}\sq*{|m_k - g_k|} &= \mathbb{E}\sq*{\left|\frac{1-\beta}{1-\beta^{k+1}}\sum_{t=0}^k \sq*{\beta^t g_{k-t}} - \frac{1-\beta}{1-\beta^{k+1}}\sum_{t=0}^k \sq*{ \beta^t g_k}\right|} & \text{since } \frac{1-\beta}{1-\beta^{k+1}}\sum_{t=0}^k \beta^t = 1\nonumber\\
&\leq \frac{1-\beta}{1-\beta^{k+1}}\sum_{t=0}^k \sq*{\beta^t \mathbb{E}\sq*{|g_{k-t} -g_{k}|}} & \nonumber\\
&\leq 2(1-\beta)\sum_{t=1}^k \beta^t \mathbb{E}\sq*{|g_{k-t} - g_{k}|} &\text{since } k\geq C \label{eq:bound_bias_interm1}
\end{align}
To proceed, we need the following lemma.
\begin{lemma}
Under Assumption~\ref{a:coordinate_lip}, for any sign vector $s\in \{-1,1\}^d$, any $x\in\R^d$ and any $\epsilon \leq \delta$
$$
\| g(x+\epsilon s) - g(x)\|_1 \leq 2\epsilon\|\vec L\|_1.
$$
\end{lemma}
\begin{proof}
By Taylor's theorem, 
$$
g(x+\epsilon s) - g(x)  = \left[\int_{t=0}^{1} H(x + t\epsilon s) dt \right]  \epsilon s. 
$$
Let $v := \text{sign}(g(x+\epsilon s) - g(x) )$, $H := \left[\int_{t=0}^{1} H(x + t\epsilon s) dt  \right]$
and moreover, use $H_+$ to denote the psd part of $H$ and $H_-$ to denote the nsd part of $H$. Namely, $H = H_+ - H_-$. 

We can write
\begin{align}
\|g(x+\epsilon s) - g(x)\|_1 &=  v^T (g(x+\epsilon s) - g(x))  = v^T H (\epsilon s) = \epsilon v^T H_+ s - \epsilon v^TH_- s \nonumber\\
&= \epsilon \langle H_+^{1/2}v, H_+^{1/2}s \rangle -  \epsilon\langle H_-^{1/2}v, H_-^{1/2}s\rangle \leq \epsilon \|H_+^{1/2}v\| \|H_+^{1/2} s\| + \epsilon\|H_-^{1/2}v\|\|H_-^{1/2}s\|. \label{eq:pf_lemma_gradientl1}
\end{align}

Note that assumption~\ref{a:coordinate_lip} implies the semidefinite ordering
$$H_+\prec \diag(\vec L) \text{ and } H_- \prec \diag(\vec L) $$
and thus
$\max\{s^T H_+s,s^TH_-s\} \leq \sum_{i=1}^d L_i = \|\vec L\|_1$ for all $s\in\{-1,1\}^d$.

The proof is complete by observing that both $v$ and $s$ are sign vectors in \eqref{eq:pf_lemma_gradientl1}.
\end{proof}


Using the above lemma and the fact that our update rules are always following some sign vectors with learning rate smaller than $\delta$, we have 
\begin{align}
\|g_{k-t} - g_{k}\|_1 &\leq  \sum_{l=0}^{t-1} \| g_{k-l} - g_{k-l-1}\|_1  \nonumber\\
&\leq   2\|\vec L\|_1\sum_{l=0}^{t-1}\delta_{k-l-1}  \leq   \sum_{l=0}^{t-1}\frac{2\|\vec L\|_1 \delta}{\sqrt{k-l}} \nonumber\\
&\leq 2\|\vec L\|_1 \delta \int_{k-t}^k \frac{dx}{\sqrt{x}}  = 4\|\vec L\|_1 \delta \br*{\sqrt{k} - \sqrt{k-t}}\nonumber\\
&\leq 4\|\vec L\|_1 \delta\sqrt{k} \br*{1 - \sqrt{1-\frac{t}{k}}}\nonumber\\
&\leq  4\|\vec L\|_1 \delta \frac{t}{\sqrt{k}} & \text{for }x\geq0\text{, } 1 - x \leq \sqrt{1 - x} \label{eq:bound_bias_interm11}
\end{align}

It follows that
\begin{align}
\sum_i \E[|m_k[i]-g_k[i]|] &= \E\left[ \|m_k - g_k\|_1 \right]&\nonumber\\
 &\leq   2(1-\beta)\sum_{t=1}^k \beta^t \mathbb{E}\|g_{k-t}-g_k\|_1 &\text{Apply \eqref{eq:bound_bias_interm1}}\nonumber\\
&\leq \frac{8(1-\beta)\|\vec L\|_1 \delta}{\sqrt{k}}  \sum_{t=1}^\infty t\beta^t & \text{Apply \eqref{eq:bound_bias_interm11} and extend sum to } \infty\nonumber\\
&\leq\frac{8(1-\beta)\|\vec L\|_1 \delta}{\sqrt{k}} \frac{\beta}{(1-\beta)^2}&\text{derivative of geometric progression}\nonumber\\
&\leq \frac{16\|\vec L\|_1\delta}{\sqrt{k+1}}\frac{\beta}{1-\beta}  & \text{for $k\geq1$, $\sqrt{\frac{k+1}{k}}\leq2$} \label{eq:bias_bound}
\end{align}


Substitute \eqref{eq:std_bound} into \eqref{eq:bias-var-decomp}, sum both sides over $i$ and then further plug in \eqref{eq:bias_bound} we get the statement in the lemma.
\end{proof}

The proof of Theorem~\ref{thm:signum} now follows in a straightforward manner. Note that Lemma \ref{lem:gradient_approx} only kicks in after a warmup period of $C$ iterations, with $C$ as specified in Theorem \ref{thm:signum}. In theory it does not matter what you do during this warmup period, provided you accumulate the momentum as normal and take steps according to the prescribed learning rate and mini-batch schedules. One option is to just stay put and not update the parameter vector for the first $C$ iterations. This is wasteful since no progress will be made on the objective. A better option in practice is to take steps using the sign of the stochastic gradient (i.e. do \signSGD) instead of \Signum{} during the warmup period.\\

\begin{proof}[Proof of Theorem~\ref{thm:signum}]

Substitute Lemma~\ref{lem:gradient_approx} as $\xi(k)$ into Lemma~\ref{lem:master} and check that 
$
\xi(k)= O(1/\sqrt{k}), \delta_k=O(1/\sqrt{k}), \min \delta_k = \delta/\sqrt{K}, C\ll K,
$
and in addition, we note that by the increasing minibatch size $N_K = O(K^2)$. Substitute $K =O(\sqrt{N})$ and take the square on both sides of the inequality. (We can take the $\min_{k}$ out of the square since all the arguments are nonnegative and $(\cdot)^2$ is monotonic on $\R_+$).
\end{proof}



\end{document}